\documentclass[a4paper,11pt]{article}
\pdfoutput=1
\usepackage{a4wide}
\usepackage{amsmath,amsfonts,amssymb,amsthm}
\usepackage[colorlinks,citecolor=blue]{hyperref}
\usepackage{enumitem}
\usepackage{bbm}
\usepackage{natbib}
\usepackage{microtype}
\usepackage{graphicx}

\setcitestyle{square}

\newtheorem{theorem}{Theorem}[section]
\newtheorem{lemma}[theorem]{Lemma}
\newtheorem{proposition}[theorem]{Proposition}

\theoremstyle{definition}

\newtheorem{example}[theorem]{Example}
\newtheorem{remark}[theorem]{Remark}

\DeclareMathOperator*{\argmin}{argmin}

\DeclareMathOperator*{\esssup}{ess\,sup}

\numberwithin{equation}{section}
\numberwithin{table}{section}
\numberwithin{figure}{section}

\def \bE {\mathbb{E}}

\def \bN {\mathbb{N}}

\def \bR {\mathbb{R}}

\def \bZ {\mathbb{Z}}

\def \cA {\mathcal{A}}
\def \cB {\mathcal{B}}

\def \cD {\mathcal{D}}
\def \cE {\mathcal{E}}
\def \cF {\mathcal{F}}

\def \cH {\mathcal{H}}
\def \cI {\mathcal{I}}

\def \cL {\mathcal{L}}

\def \cN {\mathcal{N}}
\def \cO {\mathcal{O}}
\def \cP {\mathcal{P}}

\def \cS {\mathcal{S}}
\def \cT {\mathcal{T}}

\def \cW {\mathcal{W}}

\def \NN {\mathcal{NN}}

\def \Bi {{\boldsymbol{i}}}

\def \pwl {\,{\rm PWL}\,}
\def \Id {\,{\rm Id}\,}
\def \ind {\,{\rm ind}\,}
\def \Bin {\,{\rm Bin}\,}
\def \sgn {\,{\rm sgn}\,}
\def \Pdim {\,{\rm Pdim}\,}

\begin{document}
\title{On the optimal approximation of Sobolev and Besov functions using deep ReLU neural networks}
\author{
Yunfei Yang \thanks{School of Mathematics (Zhuhai) and Guangdong Province Key Laboratory of Computational Science, Sun Yat-sen University, Zhuhai, P.R. China. E-mail: \href{mailto:yangyunfei@mail.sysu.edu.cn}{yangyunfei@mail.sysu.edu.cn}.}
}
\date{}
\maketitle

\begin{abstract}

This paper studies the problem of how efficiently functions in the Sobolev spaces $\mathcal{W}^{s,q}([0,1]^d)$ and Besov spaces $\mathcal{B}^s_{q,r}([0,1]^d)$ can be approximated by deep ReLU neural networks with width $W$ and depth $L$, when the error is measured in the $L^p([0,1]^d)$ norm. This problem has been studied by several recent works, which obtained the approximation rate $\mathcal{O}((WL)^{-2s/d})$ up to logarithmic factors when $p=q=\infty$, and the rate $\mathcal{O}(L^{-2s/d})$ for networks with fixed width when the Sobolev embedding condition $1/q -1/p<s/d$ holds. We generalize these results by showing that the rate $\mathcal{O}((WL)^{-2s/d})$ indeed holds under the Sobolev embedding condition. It is known that this rate is optimal up to logarithmic factors. The key tool in our proof is a novel encoding of sparse vectors by using deep ReLU neural networks with varied width and depth, which may be of independent interest.


\end{abstract}

\section{Introduction}

Deep learning methods have made remarkable achievements in many fields such as computer vision, natural language processing and scientific computing \citep{lecun2015deep,raissi2019physics}. The breakthrough of deep learning has motivated a lot of research on the theoretical understanding of why deep neural networks are so powerful in applications. One of the key reasons for the great success of neural networks is their ability to effectively approximate many complex nonlinear functions. The well-known universal approximation theorem \citep{cybenko1989approximation,hornik1991approximation} shows that a neural network with one hidden layer can approximate any continuous functions on compact sets up to any prescribed accuracy. In recent studies, approximation rates of deep neural networks have been derived for many function spaces, such as continuous functions \citep{shen2019nonlinear,shen2020deep,shen2022optimal,yarotsky2018optimal}, piecewise smooth functions \citep{petersen2018optimal}, Sobolev functions \citep{lu2021deep,yarotsky2017error,yarotsky2020phase} and Besov functions \citep{siegel2023optimal,suzuki2019adaptivity}.

In this paper, we are interested in the problem of how efficiently functions in Sobolev or Besov spaces can be approximated by deep neural networks with the ReLU activation function \citep{nair2010rectified}. To be concrete, our goal is to estimate the $L^p$-approximation rate
\[
\sup_{f\in \cF} \inf_{g\in \NN(W,L)} \|f-g\|_{L^p([0,1]^d)}
\]
for the function class $\NN(W,L)$ of deep neural networks with width $W$ and depth $L$, when the target function class $\cF$ is the unit ball of the Sobolev space $\cW^{s,q}([0,1]^d)$ or the Besov space $\cB^s_{q,r}([0,1]^d)$ (see Section \ref{sec: main results} below for precise definitions of these spaces and the neural network class). This problem has been studied by several recent works. For the Sobolev space $\cW^{s,\infty}([0,1]^d)$, \citet{yarotsky2020phase} proved the rate $\cO((L/\log L)^{-2s/d})$ when the network width $W$ is finite. This result was improved to $\cO((W^2L^2 \log W)^{-s/d})$ for $0<s\le 1$ in \citet{shen2022optimal} and $\cO((WL/(\log W \log L))^{-2s/d})$ for all $s>0$ in \citet{lu2021deep}. For general Sobolev and Besov spaces, \citet{siegel2023optimal} obtained the rate $\cO(L^{-2s/d})$ for networks with finite width under the strict Sobolev embedding condition $1/q -1/p<s/d$, which guarantees that $\cW^{s,q}([0,1]^d)$ and $\cB^s_{q,r}([0,1]^d)$ are compactly embedded in $L^p([0,1]^d)$. We generalize these results by proving that, under the condition $1/q -1/p<s/d$,
\begin{equation}\label{main bound}
\sup_{\|f\|_{\cW^{s,q}([0,1]^d)}\le 1} \inf_{g\in \NN(W,L)} \|f-g\|_{L^p([0,1]^d)} \le C (WL)^{-2s/d},
\end{equation}
for sufficiently large width $W$ and depth $L$. Similar result also holds when the Sobolev space is replaced by the Besov space. It is known that the rate $\cO((WL)^{-2s/d})$ is optimal up to logarithmic factors \citep{lu2021deep,siegel2023optimal}.

As pointed out by \citet{siegel2023optimal}, the main technical difficulty in proving (\ref{main bound}) is to deal with the case when $p>q$. Because, when $p\le q$, the approximation rate (\ref{main bound}) can be achieved by classical linear approximation methods using piecewise polynomials, while for $p>q$, nonlinear adaptive methods are required \citep{devore1998nonlinear}. Thus, in the nonlinear regime $p>q$, one needs to use piecewise polynomials on an adaptive non-uniform grid, which cannot be handled by the methods in \citet{lu2021deep,shen2022optimal}. To overcome this difficulty, \citet{siegel2023optimal} used a novel bit-extraction technique to optimally encodes sparse vectors using deep ReLU networks with fixed width. One of our main technical contributions is a generalization of this result, presented in Theorem \ref{rep of vector}, to the case when both the width and depth vary.

The rest of this paper is organized as follows. Section \ref{sec: main results} presents our main results on the approximation of Sobolev and Besov functions. The proof is given in Section \ref{sec: main proof}. We illustrate how to apply the approximation results to derive convergence rates for learning algorithms in Section \ref{sec: applications}. Finally, we remark that, unless otherwise specified, we will use $C$ to denote constants which may change from line to line. This convention is standard and convenient in analysis. The constants $C$ may depend on some other parameters and this dependence will be made clear from the context.

\section{Main approximation results}\label{sec: main results}

Let us begin with a formal definition of the neural network classes used in this paper. Given $L,N_1,\dots, N_L \in \bN$, we consider the mapping $g:\bR^{d} \to \bR^{k}$ that can be parameterized by a fully connected ReLU neural network of the following form
\begin{align*}
g_0(x) &= x, \\
g_{\ell+1}(x) &= \sigma(A_\ell g_\ell (x) + b_\ell), \quad \ell = 0,1,\dots,L-1, \\
g(x) &= A_L g_L(x) +b_L,
\end{align*}
where $A_\ell \in \bR^{N_{\ell+1}\times N_{\ell}}$, $b_\ell\in \bR^{N_{\ell+1}}$ with $N_0 =d$ and $N_{L+1} =k$. The activation function $\sigma(t) := \max\{t,0\}$ is the Rectified Linear Unit function (ReLU) and it is applied component-wisely. We remark that there is no activation function in the output layer, which is the usual convention in applications. The numbers $W:=\max\{N_1,\dots,N_L\}$ and $L$ are called the width and depth (the number of hidden layers) of the neural network, respectively. We denote by $\cN\cN_{d,k}(W,L)$ the set of mappings that can be parameterized by ReLU neural networks with width $W$ and depth $L$. When the input dimension $d$ and the output dimension $k$ are clear from contexts, we simplify the notation to $\cN\cN(W,L)$ for convenience. We will give some basic properties of the neural network classes in Proposition \ref{basic constr} below. 

The purpose of this paper is to study the approximation of Sobolev and Besov functions by using neural networks. Let us recall the definitions of Sobolev and Besov spaces for the reader's convenience \citep{adams2003sobolev,dinezza2012hitchhiker,triebel1992theory}. Let $\Omega \subseteq \bR^d$ be a bounded domain, which we take to be the unit cube $\Omega=[0,1]^d$ in the following. For $1\le q\le \infty$, we denote by $L^q(\Omega)$ the set of functions $f$ whose $L^q$ norm on $\Omega$ is finite. In other words, when $q<\infty$,
\[
\|f\|_{L^q (\Omega)}^q = \int_{\Omega} |f(x)|^q dx <\infty.
\]
When $q=\infty$, we have the standard modification $\|f\|_{L^\infty(\Omega)} = \esssup_{x\in \Omega} |f(x)|<\infty$. For a positive integer $s\in \bN$, we say $f\in L^q(\Omega)$ is in the Sobolev space $\cW^{s,q}(\Omega)$ if it has weak derivatives of order $s$ and 
\begin{equation}\label{sobolev norm}
\| f\|_{\cW^{s,q}(\Omega)}^q := \|f\|_{L^q(\Omega)}^q + |f|_{\cW^{s,q}(\Omega)}^q < \infty,
\end{equation}
where the Sobolev semi-norm $|f|_{\cW^{s,q}(\Omega)}$ is defined by
\[
|f|_{\cW^{s,q}(\Omega)}^q := \sum_{|\gamma|=s} \|\partial^\gamma f\|_{L^q(\Omega)}^q.
\]
Here $\gamma = (\gamma_i)_{i=1}^d$ with $\gamma_i\in \bN_0:=\bN\cup \{0\}$ is a multi-index with total degree $|\gamma|= \sum_{i=1}^d \gamma_i$ and we make the usual modification when $q=\infty$. When $s$ is not an integer, we modify the Sobolev semi-norm by generalizing the H\"older condition
\[
|f|_{\cW^{s,q}(\Omega)}^q := \sum_{|\gamma|=\lfloor s \rfloor} \left( \|\partial^\gamma f\|_{L^q(\Omega)}^q + \int_{\Omega} \int_{\Omega} \frac{|\partial^\gamma f(x) - \partial^\gamma f(y)|^q}{|x-y|^{d+(s-\lfloor s \rfloor)q}} dxdy \right),
\]
when $q<\infty$ and 
\[
|f|_{\cW^{s,\infty}(\Omega)} := \sum_{|\gamma|=\lfloor s \rfloor} \left( \|\partial^\gamma f\|_{L^\infty(\Omega)} + \esssup_{x,y\in \Omega} \frac{|\partial^\gamma f(x) - \partial^\gamma f(y)|}{|x-y|^{s-\lfloor s \rfloor}} \right).
\]
Then, we can define the space $\cW^{s,q}(\Omega)$ via the norm (\ref{sobolev norm}) for non-integer $s$. These spaces are also called Sobolev-Slobodeckij spaces in the literature.

Next, we define the Besov spaces through the moduli of smoothness. For $k\in \bN$, the $k$-th order modulus of smoothness of a function $f\in L^q(\Omega)$ is defined as
\[
\omega_k(f,t)_q = \sup_{|h|<t} \|\Delta_h^k f \|_{L^q(\Omega_{kh})},
\]
where $h\in \bR^d$, $\Omega_{kh} = \{x\in \Omega:x+kh\in\Omega\}$ and the $k$-th order difference $\Delta_h^k$ is given by
\[
\Delta_h^k f(x) = \sum_{j=0}^k (-1)^j \binom{k}{j} f(x+jh).
\]
Let $s>0$, $1\le q,r\le \infty$ and fix an integer $k>s$. The Besov space $\cB^s_{q,r}(\Omega)$ is defined through the norm 
\[
\|f\|_{\cB^s_{q,r}(\Omega)} := \|f\|_{L^q(\Omega)} + |f|_{\cB^s_{q,r}(\Omega)},
\]
where the Besov semi-norm is given by
\[
|f|_{\cB^s_{q,r}(\Omega)} :=
\begin{cases}
( \int_0^\infty (t^{-s} \omega_k(f,t)_q)^r t^{-1}dt )^{1/r}, &\mbox{if } 1\le r<\infty,\\
\sup_{t>0} t^{-s} \omega_k(f,t)_q, &\mbox{if } r=\infty.
\end{cases}
\]
It is possible to show that different choices of $k>s$ give equivalent norms \citep{devore1993constructive}. Thus, one can simply choose $k=\lfloor s \rfloor +1$. Sobolev and Besov spaces are closely related to each other (see \citet{triebel1992theory} for instance). We remark that, the continuous embedding $\cB^s_{q,r_1}(\Omega) \hookrightarrow \cB^s_{q,r_2}(\Omega)$ holds for $1\le r_1\le r_2\le \infty$. When $s$ is not an integer, it holds that $\cB^s_{q,q}(\Omega) = \cW^{s,q}(\Omega)$ with equivalent norm. When $s\in \bN$, we have the continuous embedding $\cB^s_{q,q}(\Omega) \hookrightarrow \cW^{s,q}(\Omega) \hookrightarrow \cB^s_{q,2}(\Omega)$ if $q\le 2$ and the reverse $\cB^s_{q,2}(\Omega) \hookrightarrow \cW^{s,q}(\Omega) \hookrightarrow \cB^s_{q,q}(\Omega)$ if $q\ge 2$. In particular, we have $\cB^s_{2,2}(\Omega) = \cW^{s,2}(\Omega)$.

Our goal is to quantify how efficiently the neural network class $\NN(W,L)$ can approximate functions in $\cW^{s,q}(\Omega)$ or $\cB^s_{q,r}(\Omega)$, where $\Omega=[0,1]^d$ is the unit cube. If the approximation error is measured in the $L^p(\Omega)$-norm, then it is necessary to assume that $\cW^{s,q}(\Omega)$ or $\cB^s_{q,r}(\Omega)$ is contained in $L^p(\Omega)$. Because, we cannot get any approximation rate for $f\notin L^p(\Omega)$, since ReLU neural networks can only represent continuous functions. Indeed, we assume that the following strict Sobolev embedding condition holds
\[
\frac{1}{q} - \frac{1}{p}< \frac{s}{d},
\]
which guarantees that the embeddings $\cW^{s,q}(\Omega) \hookrightarrow L^p(\Omega)$ and $\cB^s_{q,r}(\Omega) \hookrightarrow L^p(\Omega)$ are compact. Note that, on the boundary condition $1/q-1/p=s/d$, whether the embeddings hold depends on the precise values of $s,p,q$ and $r$. Thus, this boundary
case is much more subtle and we do not study it in this work. We present our main results in the following two theorems and defer the proof to Section \ref{sec: main proof}.

\begin{theorem}\label{app of sobolev}
Let $0<s<\infty$ and $1\le p,q\le \infty$. If $1/q - 1/p <s/d$, then for sufficiently large $W,L\in \bN$,
\[
\inf_{g\in \NN(W,L)} \|f-g\|_{L^p([0,1]^d)} \le C \|f\|_{\cW^{s,q}([0,1]^d)} (WL)^{-2s/d},
\]
for some constant $C$ depending on $s,p,q$ and $d$.
\end{theorem}

\begin{theorem}\label{app of besov}
Let $0<s<\infty$ and $1\le r,p,q\le \infty$. If $1/q - 1/p <s/d$, then for sufficiently large $W,L\in \bN$,
\[
\inf_{g\in \NN(W,L)} \|f-g\|_{L^p([0,1]^d)} \le C \|f\|_{\cB^s_{q,r}([0,1]^d)} (WL)^{-2s/d},
\]
for some constant $C$ depending on $s,r,p,q$ and $d$.
\end{theorem}

The approximation rate $\cO((WL)^{-2s/d})$ is known to be optimal up to logarithmic factors \citep{shen2020deep,siegel2023optimal,yang2022approximation,yarotsky2017error}. Specifically, \citet[Theorem 3]{siegel2023optimal} showed the existence of $f$ with $\|f\|_{\cW^{s,q}([0,1]^d)} \le 1$ and $\|f\|_{\cB^s_{q,r}([0,1]^d)} \le 1$ such that 
\begin{equation}\label{lower bound}
\inf_{g\in \NN(W,L)} \|f-g\|_{L^p([0,1]^d)} \ge C(p,d,s) \min \{ W^2L^2\log(WL), W^3L^2 \}^{-s/d}.
\end{equation}
In particular, when the width $W$ is bounded, we get the optimal rate $\cO(L^{-2s/d})$, which has been proven by \citet{siegel2023optimal}. We remark that the lower bound (\ref{lower bound}) is derived from upper bounds for the VC-dimension of ReLU neural networks (see \citep{bartlett2019nearly} and inequality (\ref{Pdim bound}) below, note that VC-dimension is not larger than pseudo-dimension). 

There is a series of works trying to characterize the approximation rates for deep ReLU networks in terms of the number of nonzero parameters, see \citep{guehring2019error,petersen2018optimal,suzuki2019adaptivity,yarotsky2017error} for instance. Many of these results can be obtained by using Theorems \ref{app of sobolev} and \ref{app of besov}. Indeed, for fully connected networks with depth $L\ge 2$, the number of parameters in the network is $N = \cO(W^2L)$. Our results give the approximation rate $\cO((WL)^{-2s/d}) \le \cO(N^{-s/d})$ for sufficiently wide and deep networks. This rate can be improved to $\cO(N^{-2s/d})$ if the width is bounded so that $N = \cO(L)$.

Let us denote by $W^*$ and $L^*$ the minimal width and depth respectively such that the approximation rate $\cO((WL)^{-2s/d})$ holds in Theorems \ref{app of sobolev} and \ref{app of besov}. Clearly, $W^*$ and $L^*$ depend on $s,r,p,q$ and $d$. Although we do not try to estimate the values of $W^*$ and $L^*$ in this paper, one can get some information from related works. For instance, the result of \citet{siegel2023optimal} implies that $W^*\le 25d+31$, while \citet{hanin2017approximating} showed that the set of ReLU neural networks with width $W\le d$ is not dense in $C(\Omega)$. Thus, the minimal width $W^*$ is linear on the dimension $d$. The recent work \citep{liu2024relu} proved that this result can be further improved to $W^*=d+\cO(1)$ in certain cases. The situation is more complicated for the minimal depth $L^*$. It was shown by \citet[Theorem 6]{yarotsky2017error} and \citet[Theorem 4]{safran2017depth} that the approximation error in $L^2([0,1]^d)$ for any nonlinear function $f\in C^2([0,1]^d)$ is lower bounded by $C_f W^{-2L}$ for some constant $C_f>0$. This lower bound implies that we must have $L^* \ge s/d$ for $p\ge 2$, in order to get the rate $\cO(W^{-2s/d})$ for finite depth. On the other hand, \citet[Corollary 1.2]{lu2021deep} proved that, when the target function class is $C^s([0,1]^d)$ with $s\in \bN$, the depth $L=108s^2+2d$ is sufficient to obtain a slightly weaker bound $\cO((W/\log W)^{-2s/d})$. It would be an interesting problem to give more precise estimations for $W^*$ and $L^*$. We think the tools developed in Section \ref{sec: main proof} would be helpful for this problem.

\section{Applications to machine learning}\label{sec: applications}

This section illustrates how to apply Theorems \ref{app of sobolev} and \ref{app of besov} to study machine learning problems. We will use our approximation results to derive new learning rates for the least squares estimator in the nonparametric regression setting, which has attracted a lot of attention in recent research \citep{chen2022nonparametric,kohler2021rate,nakada2020adaptive,schmidthieber2020nonparametric,suzuki2019adaptivity}. Although this paper focuses on the regression problem, we remark that similar analysis can be applied to study other learning problems, such as classification \citep{kim2021fast,yang2024rates}, solving partial differential equations \citep{duan2022convergence,lu2022machine} and distribution learning by diffusion modeling \citep{oko2023diffusion}.

Suppose we have a data set of $n\ge 2$ samples $\cD_n = \{(X_i,Y_i)\}_{i=1}^n$, which are independent and identically distributed as a $\bR^d \times \bR$-valued random vector $(X,Y)$. Let $\mu$ be the marginal distribution of the covariate $X$. We assume that $\mu$ is supported on $[0,1]^d$ and absolutely continuous with respect to the Lebesgue measure with density $p_X$ which satisfies $0\le p_X(x)\le C<\infty$ on $[0,1]^d$. The goal of nonparametric regression problem is to estimate the so-called regression function $f(x) = \bE[Y|X=x]$ from the observed data $\cD_n$. One of the most popular estimators is the least squares
\begin{equation}\label{least squares}
\widehat{h}_n \in \argmin_{h\in \cH} \frac{1}{n} \sum_{i=1}^n (h(X_i)- Y_i)^2,
\end{equation}
where $\cH$ is a suitably chosen hypothesis class. For simplicity, we assume here and in the sequel that the minimum above indeed exists. In deep learning, the function class $\cH$ is parameterized by deep neural networks. So, we consider the case that $\cH =\NN(W,L)$, where the width $W$ and depth $L$ depend on the sample size $n$ so that we can obtain convergence rate for the estimator $\widehat{h}_n$. The performance of the estimation is measured by the expected risk
\[
\cL(\widehat{h}_n) := \bE_{(X,Y)} [(\widehat{h}_n(X)-Y)^2].
\]
It is equivalent to evaluating the estimator by the excess risk
\[
\|\widehat{h}_n - f\|_{L^2(\mu)}^2 = \cL(\widehat{h}_n) - \cL(f).
\]

In principle, the excess risk can be divided into two components: the approximation error due to the representational capacity of the model and the sample error (also called estimation error) due to the fact that we only have finite samples. We can estimate the approximation error by Theorems \ref{app of sobolev} and \ref{app of besov}, while the sample error is often bounded by the covering number or pseudo-dimension of the model \citep{mohri2018foundations}. Recall that the pseudo-dimension $\Pdim(\cH)$ of a real-valued function class $\cH$ defined on $[0,1]^d$ is the largest integer $m$ for which there exist points $x_1,\dots,x_m \in [0,1]^d$ and constants $c_1,\dots,c_m\in \bR$ such that
\[
\left|\{ \sgn(h(x_1)-c_1),\dots,\sgn(h(x_m)-c_m): h\in \cH \}\right| =2^m.
\]
\citet[Theorems 7 and 10]{bartlett2019nearly} showed that 
\begin{equation}\label{Pdim bound}
\Pdim(\NN(W,L)) \le C \min\{ W^2L^2\log(WL), W^3L^2 \}.
\end{equation}

In the statistical analysis of learning algorithms, we often require that the hypothesis class is uniformly bounded. We define the truncation operator $\cT_B$ with level $B>0$ for real-valued functions $h$ as
\[
\cT_Bh(x) := 
\begin{cases}
h(x) &\quad \mbox{if }|h(x)|\le B, \\
\sgn(h(x)) B &\quad \mbox{if } |h(x)|> B.
\end{cases}
\]
For a function class $\cH$ containing real-valued functions, we denote $\cT_B \cH := \{\cT_Bh: h\in \cH\}$. Note that the truncation can be implemented by the ReLU neural network $\sigma(t) - \sigma(-t) -\sigma(t-B) +\sigma(-t-B) \in \NN(4,1)$. Thus, the truncation of a neural network is also a neural network and $\cT_B \NN(W,L) \subseteq \NN(\max\{W,4\},L+1)$ by Proposition \ref{basic constr} below. The next theorem provides a convergence rate for the truncated least squares $\cT_{B_n} \widehat{h}_n$, where $B_n = c\log n$ for some constant $c>0$ and the hypothesis class is a neural network.

\begin{theorem}\label{rate}
Suppose $s\in (0,\infty)$ and $q\in [1,\infty]$ satisfy $1/q-1/2<s/d$. Let $c_1,c_2,c_3,c_4>0$ below be constants. Assume that the distribution of $(X,Y)$ satisfies $\bE[\exp(c_1 Y^2)]<\infty$ and that the regression function $f\in \cW^{s,q}([0,1]^d) \cap L^\infty([0,1]^d)$ with $\|f\|_{\cW^{s,q}([0,1]^d)}\le 1$ and $\|f\|_{L^\infty([0,1]^d)} \le B$ for some $B\ge 1$. Let $\widehat{h}_n$ be the least squares estimator (\ref{least squares}) with $\cH=\NN(W_n,L_n)$, where $W_n \ge W^*$ and $L_n \ge L^*$ so that Theorem \ref{app of sobolev} can be applied with $p=2$. If $B_n = c_2\log n$ and 
\[
c_3 n^{\frac{d}{2d+4s}} \le W_nL_n \le c_4 n^{\frac{d}{2d+4s}},
\]
then we have
\[
\bE_{\cD_n} \|\cT_{B_n}\widehat{h}_n-f\|_{L^2(\mu)}^2 \le C n^{-\frac{2s}{d+2s}} (\log n)^4,
\]
where  $\bE_{\cD_n}$ indicates the  expectation with respect to the training data $\cD_n$ and $C$ is a constant independent of the regression function $f$ and the sample size $n$.
\end{theorem}
\begin{proof}
The proof is similar to \citet[Theorem 4.2]{yang2024optimal}. By the assumption on the distribution of $(X,Y)$, we can apply the result of \citet[Supplement B, Lemma 18]{kohler2021rate}, which shows that the error can be bounded as
\[
\bE_{\cD_n} \|\cT_{B_n}\widehat{h}_n-f\|_{L^2(\mu)}^2 \le C\cE_{gen} + 2\cE_{app},
\]
where $\cE_{gen}$ denotes the generalization bound based on metric entropy and $\cE_{app}$ is the approximation error of the hypothesis class:
\begin{align*}
\cE_{gen} :=& \frac{(\log n)^2 \sup_{X_{1:n}\in ([0,1]^d)^n}\log (\cN(n^{-1}B_n^{-1}, \cT_{B_n}\cH,\|\cdot\|_{L^1(X_{1:n})})+1)}{n}, \\
\cE_{app} :=& \inf_{h\in \cH} \|f-h\|_{L^2(\mu)}^2.
\end{align*}
Here, $X_{1:n} =(X_1,\dots,X_n)$ denotes the sequence of sample points on $[0,1]^d$ and $\cN(\epsilon, \cT_{B_n}\cH,\|\cdot\|_{L^1(X_{1:n})})$ denotes the $\epsilon$-covering number of the function class $\cT_{B_n}\cH$ in the metric $\|h_1-h_2\|_{L^1(X_{1:n})} = \frac{1}{n}\sum_{i=1}^n|h_1(X_i)-h_2(X_i)|$. 

Since the density of $\mu$ is bounded, Theorem \ref{app of sobolev} implies that the approximation error can be bounded as 
\[
\cE_{app} = \inf_{h\in \cH} \|f-h\|_{L^2(\mu)}^2 \le C (W_nL_n)^{-\frac{4s}{d}} \le C n^{-\frac{2s}{d+2s}}.
\]
On the other hand, the classical result of \citet[Theorem 6]{haussler1992decision} shows that the covering number can be bounded by pseudo-dimension:
\[
\log \cN(\epsilon, \cT_{B_n}\cH,\|\cdot\|_{L^1(X_{1:n})}) \le C \Pdim(\cT_{B_n}\cH) \log(B_n/\epsilon).
\]
Using $\Pdim(\cT_{B_n}\cH) \le \Pdim(\cH)$ from \citet[Theorem 5]{haussler1992decision} and the pseudo-dimension bound (\ref{Pdim bound}) for neural networks, we get 
\[
\log \cN(\epsilon, \cT_{B_n}\cH,\|\cdot\|_{L^1(X_{1:n})}) \le C W_n^2 L_n^2 \log (W_nL_n) \log(B_n/\epsilon).
\]
This implies the following generalization bound
\[
\cE_{gen} \le C \frac{(\log n)^2 W_n^2 L_n^2 \log (W_nL_n) \log(nB_n^2)}{n} \le C n^{-\frac{2s}{d+2s}} (\log n)^4,
\]
which completes the proof.
\end{proof}

A completely analogous theorem holds for the unit ball of the Besov spaces, i.e. Theorem \ref{rate} holds with the Sobolev space $\cW^{s,q}([0,1]^d)$ replaced by $\cB^s_{q,r}([0,1]^d)$. It is well-known that the convergence rate $n^{-2s/(d+2s)}$ is minimax optimal for these spaces \citep{donoho1998minimax,gine2015mathematical,stone1982optimal}. Thus, deep neural networks can achieve the minimax optimal rates for Sobolev and Besov spaces up to logarithm factors. Similar result has been established in several recent works \citep{kohler2021rate,schmidthieber2020nonparametric,suzuki2019adaptivity}. For comparison, \citet{schmidthieber2020nonparametric} derive minimax optimal rates for learning composition of H\"older functions by sparse neural networks. \citet{suzuki2019adaptivity} showed that the minimax optimal rates also hold for Sobolev and Besov spaces. However, their results rely on the sparsity of neural networks and hence one need to optimize over different network architectures to obtain the optimal rates, which is hard to implement due to the unknown locations of the non-zero parameters. \citet{kohler2021rate} proved that fully connected networks are already able to achieve optimal rates for learning composition of H\"older functions. We complement their results by establishing the optimal rates for learning Sobolev and Besov functions using fully connected networks.

We remark that the above convergence rates suffer from the curse of dimensionality. In practical applications of deep learning, the data distributions are often of high-dimensional but have certain low-dimensional structure \citep{nakada2020adaptive}. For instance, a popular assumption is that the data distribution is concentrated around certain low-dimensional manifold \citep{chen2022nonparametric,jiao2023deep}. In order to deal with this case, it is necessary to generalize Theorems \ref{app of sobolev} and \ref{app of besov} to the approximation on manifolds. We leave this as an open problem for future study.

\section{Constructive proof of main approximation bounds}\label{sec: main proof}

In this section, we give our main construction and proof of Theorems \ref{app of sobolev} and \ref{app of besov}. Following the ideas in \citet{lu2021deep,shen2020deep,shen2022optimal,siegel2023optimal}, we approximate Sobolev and Besov functions by piecewise polynomials and construct deep ReLU neural networks to approximate these piecewise polynomials. To describe this construction, let us first introduce some notations, which are almost identical to those in \citet[Section 4]{siegel2023optimal}.

Throughout this section, we let $b\ge 2$ be a fixed integer unless otherwise specified (in Subsection \ref{sec: app of sobolev}) and suppress the dependence on $b$ in the following notations for convenience. Notice that it is enough to consider the approximation on the half-open cube $\Omega= [0,1)^d$. For any integer $\ell \ge 0$, we can partition $\Omega$ into $b^{dl}$ subcubes:
\begin{equation}\label{partition}
\Omega = \bigcup_{\Bi\in I_\ell} \Omega_\Bi^\ell, \quad \Omega_\Bi^\ell := \prod_{j=1}^d [b^{-\ell} \Bi_j, b^{-\ell}(\Bi_j+1)),
\end{equation}
where the $d$-dimensional multi-index $\Bi$ is in the index set $I_\ell:=\{0,\dots,b^\ell-1\}^d$. For any integer $k\ge 0$, we use $\cP_k$ to denote the space of polynomials with degree at most $k$. The space of piecewise polynomials (with degree at most $k$) subordinate to the partition (\ref{partition}) is denoted by 
\[
\cP_k^\ell := \left\{ f:\Omega \to \bR, f|_{\Omega_\Bi^\ell} \in \cP_k \mbox{ for all } \Bi\in I_\ell \right\}.
\]
Note that this space has a natural basis
\begin{equation}\label{basis}
\rho_{\ell,\Bi}^\gamma(x) := 
\begin{cases}
\prod_{j=1}^d (b^{\ell}x_j-\Bi_j)^{\gamma_j},&  x\in \Omega_\Bi^\ell,\\
0,& x\notin \Omega_\Bi^\ell,
\end{cases}
\end{equation}
where $\gamma=(\gamma_1,\dots,\gamma_d)\in \bN_0^d$ is a multi-index with $|\gamma|:=\sum_{j=1}^d \gamma_j\le k$. Thus, the space $\cP_k^\ell$ is of dimension $\binom{d+k}{k} b^{d\ell}$.

Since ReLU neural networks can only represent continuous piecewise linear functions, it is difficult to directly construct deep ReLU neural networks to approximate piecewise polynomials on the boundary of the partition (\ref{partition}). To overcome this difficulty, \citet{shen2020deep} proposed to remove an arbitrarily small region (called trifling region) from $\Omega$. Specially, given $\epsilon>0$, we define
\begin{equation}\label{good region}
\Omega_{\ell,\epsilon}:= \bigcup_{\Bi\in I_\ell} \Omega_{\Bi,\epsilon}^\ell,\quad \Omega_{\Bi,\epsilon}^\ell := \prod_{j=1}^d
\begin{cases}
[b^{-\ell} \Bi_j, b^{-\ell}(\Bi_j+1)-\epsilon), & \Bi_j<b^\ell -1,\\
[1-b^{-\ell}, 1), & \Bi_j=b^\ell -1.
\end{cases} 
\end{equation}
We will construct deep neural networks to approximate piecewise polynomials from $\cP_k^\ell$ on the good region $\Omega_{\ell,\epsilon}$ in Subsection \ref{sec: app of poly} and then apply this result to derive bounds for the approximations to Sobolev and Besov functions in Subsection \ref{sec: app of sobolev}. The trifling region can be removed by using a construction similar to the method in \citet{lu2021deep,shen2022optimal,siegel2023optimal}.
The key technical contribution of our construction is the neural network representation of vectors presented in Theorem \ref{rep of vector}, which is a generalization of \citet[Theorem 14]{siegel2023optimal}. We collect preliminary results on neural network constructions in Subsection \ref{sec: basic constructions} for the reader's convenience.

\subsection{Basic constructions of neural networks}\label{sec: basic constructions}

In this subsection, we collect several useful results on neural network constructions, which will be the building blocks of our construction of approximations to Sobolev and Besov functions. These results are well-known in the literature and we only make minor modifications. The omitted proofs are given in the Appendix for completeness.

The following proposition gives basic properties of the neural network class $\NN(W,L)$. These properties are widely used in the approximation theory of neural networks, see \citet{devore2021neural,jiao2023approximation,lu2021deep,siegel2023optimal} for instance. The proof can be found in \citet[Proposition 2.5]{jiao2023approximation} and \citet[Proposition 7]{siegel2023optimal}.

\begin{proposition}\label{basic constr}
Let $f_i\in \NN_{d_i,k_i}(W_i,L_i)$ for $i=1,\dots,n$.
\begin{enumerate}[label=\textnormal{(\arabic*)},parsep=0pt]
\item If $d_1=d_2$, $k_1=k_2$ and $W_1\le W_2$, $L_1\le L_2$, then $\cN\cN_{d_1,k_1}(W_1,L_1) \subseteq \cN\cN_{d_2,k_2}(W_2,L_2)$.

\item \textnormal{\textbf{(Composition)}} If $k_1 = d_2$, then $f_2 \circ f_1 \in \cN\cN_{d_1,k_2}(\max\{W_1,W_2\},L_1+L_2)$. The result also holds when $f_1$ or $f_2$ is affine, if we view affine maps as neural networks with width $W=0$ and depth $L=0$.

\item \textnormal{\textbf{(Concatenation)}} If $d_1=d_2$, define $f(x):=(f_1(x),f_2(x))^\top$, then $f\in \cN\cN_{d_1,k_1+k_2}(W_1+W_2,\max\{L_1,L_2\})$.

\item \textnormal{\textbf{(Summation)}} If $d_i=d$ and $k_i=k$ for all $i=1,\dots,n$, then 
\[
\sum_{i=1}^n f_i \in \cN\cN_{d,k} \left(\sum_{i=1}^n W_i,
\max_{1\le i\le n} L_i\right) \cap \NN_{d,k}\left(\max_{1\le i\le n} W_i+2d+2k,\sum_{i=1}^n L_i\right).
\]
\end{enumerate}
\end{proposition}

Note that Proposition \ref{basic constr} can be applied recursively to construct new neural networks. It is easy to see that a network can be applied to only a few components of its input, because we can use an affine map to select the coordinates. Since the identity map $\Id(x) =x=\sigma(x)-\sigma(-x)\in \NN_{1,1}(2,1)$, whose width can be reduced to one when the sign of $x$ is known, we can ``memorize'' some components of the input and intermediate outputs of a network by concatenation. We will use these facts without comment in the
rest of the paper.

We remark that, in Part (4) of Proposition \ref{basic constr}, we have two ways to construct a neural network which represents the sum of a collection of smaller networks. The first is to concatenate the networks in parallel and compute the sum in the output layer. The second is to compute the sum in a sequential way, in which we use $2d$ neurons to memorize the input and $2k$ neurons to memorize the partial sum. We will mainly use the first construction in the rest of the paper. When it is necessary to use the second construction, we will mention it explicitly.

Recall that ReLU neural networks can only represent continuous piecewise linear functions. For $n\in \bN$, we use $\pwl(n)$ to denote the set of continuous piecewise linear functions $g:\bR\to \bR$ with at most $n$ pieces, that is, there exists at most $n+1$ points $-\infty = t_0 \le t_1 \le \dots\le t_n=\infty$ such that $g$ is linear on the interval $(t_{i-1},t_i)$ for all $i=1,\dots,n$. The points, where $g$ is not differentiable, are called breakpoints of $g$. The next lemma shows that we can represent any functions in $\pwl(n)$ by ReLU networks with $\cO(n)$ parameters.

\begin{lemma}\label{PWL rep}
Let $W,L,n \in \bN$ and $g\in \pwl(n+1)$.
\begin{enumerate}[label=\textnormal{(\arabic*)},parsep=0pt]
\item It holds that $g\in \NN(n+1,1)$. If $g'(t) =0$ for sufficiently small $t$, then $g\in \NN(n,1)$.

\item Assume that the breakpoints of $g$ are in the bounded interval $[\alpha,\beta]$. If $n\le 6W^2L$, then there exists $f\in \NN(6W+2,2L)$ such that $f=g$ on $[\alpha,\beta]$.
\end{enumerate}
\end{lemma}

The key to obtain sharp approximation results for deep ReLU networks is the bit extraction technique, which was introduced to lower bound the VC dimension of neural networks with piecewise polynomial activation \citep{bartlett1998almost,bartlett2019nearly} and used by \citet{yarotsky2018optimal,shen2022optimal} to derive optimal approximation rates for deep ReLU networks. To present the technique, let us denote the binary representation by 
\begin{equation}\label{binary rep}
\Bin x_mx_{m-1}\cdots x_0.x_{-1} \cdots x_{-n} := \sum_{i=-n}^m 2^i x_i,
\end{equation}
for $x_i\in \{0,1\}$, $i=-n,\dots,m$. The next lemma is a minor modification of \citet[Lemma 13]{bartlett2019nearly}.

\begin{lemma}\label{bit extraction}
Let $m,n\in \bN$ satisfy $m\le n$. There exists $f_{n,m}\in\NN_{1,m+1}(2^{m+2}+1,1)$ such that, for any $x=\Bin 0.x_1 \cdots x_n$ with $x_i\in \{0,1\}$, we have
\[
f_{n,m}(x) = (x_1,\dots,x_m, \Bin 0.x_{m+1}\cdots x_n)^\top.
\]
Moreover, for any $L\in \bN$, there exists $f_{n,m,L}\in\NN_{1,2}(2^{\lceil m/L \rceil+2}+2,L)$ such that
\[
f_{n,m,L}(x) = (\Bin x_1\cdots x_m.0, \Bin 0.x_m\cdots x_n)^\top.
\]
\end{lemma}

Using similar idea as the bit extraction technique, we can construct deep neural networks to compute the index $\Bi$ of the good region $\Omega_{\Bi,\epsilon}^\ell$ defined by (\ref{good region}).

\begin{lemma}\label{app partition}
Let $\ell \in \bN_0$ and $0<\epsilon<b^{-\ell}$. For any $L\in \bN$, there exists $q_d \in \NN_{d,1}(2db^{\lceil \ell/L \rceil},L)$ such that
\[
q_d(x) = \ind(\Bi) := \sum_{j=1}^d b^{\ell(j-1)} \Bi_j,\quad \forall x\in \Omega_{\Bi,\epsilon}^\ell.
\]
\end{lemma}

Finally, in order to remove the trifling region, we will need the following technical construction from \citet[Corollary 13]{siegel2023optimal}, which gives a network that selects any order statistic.

\begin{proposition}\label{select median}
Let $d=2^k$ for some $k\in \bN$. For each integer $1\le j\le d$, there exists $\psi_j\in \NN_{d,1}(4d,k(k+1)/2)$ such that $\psi_j(x) = x_{(j)}$, where $x_{(j)}$ is the $j$-th largest entry of $x\in \bR^d$.
\end{proposition}

\subsection{Representation of vectors}\label{sec: rep of sparse vec}

This subsection gives the main technical construction for the proof of Theorems \ref{app of sobolev} and \ref{app of besov}. We consider the problem of how efficiently deep ReLU neural networks can represent integer vectors. This problem has also been studied by \citet{siegel2023optimal}, which gave sharp result for networks with constant width. We give a generalization of this result in the next theorem, which pays more attention to the trade-off between width and depth.

\begin{theorem}\label{rep of vector}
Let $N,M \in \bN$ and $x=(x_1,\dots,x_N)^\top\in \bZ^N$ satisfy $\|x\|_1\le M$.
\begin{enumerate}[label=\textnormal{(\arabic*)},parsep=0pt]
\item If $N\ge M$, then for any $S,T \in \bN$, there exists $g\in \NN(W,L)$ with 
\[
W= 22 \max \left\{ \left\lceil \frac{\sqrt{M}}{S \sqrt{T+2}} \right\rceil, \left\lceil \left(\frac{N}{M}\right)^{1/T} \right\rceil \right\} +10, \quad L= 4S(T+2),
\]
such that $g(n) = x_n$ for $n=1,\dots,N$.

\item If $N\le M$, then for any $S,T \in \bN$, there exists $g\in \NN(W,L)$ with 
\[
W= 22 \max \left\{ \left\lceil \frac{M}{S \sqrt{N(T+2)}} \right\rceil, \left\lceil \left(\frac{M}{N}\right)^{1/T} \right\rceil \right\} +12, \quad L= 4 \left\lceil \frac{SN}{M} \right\rceil (T+2),
\]
such that $g(n) = x_n$ for $n=1,\dots,N$.
\end{enumerate}
\end{theorem}

Before proving this theorem, let us make a short discussion on the result. We denote the set of integer vectors which we wish to encode by
\begin{equation}\label{S_NM}
\cS_{N,M} := \left\{ x\in \bZ^N: \|x\|_1 \le M \right\}.
\end{equation}
As shown by \citet{siegel2023optimal}, the cardinality of this set satisfies
\[
\log_2 |\cS_{N,M}| \le C
\begin{cases}
M(1+\log_2(N/M) ), &\mbox{if }N\ge M, \\
N(1+\log_2(M/N) ), &\mbox{if }N\le M.
\end{cases}
\]
This bound also gives an estimate on the number of bits required to encode the set $\cS_{N,M}$. Note that we can use the parameters $S$ and $T$ to tune the size of network in Theorem \ref{rep of vector}. By choosing $T = \lceil\log_2(N/M) \rceil$ if $N\ge M$ and $T = \lceil\log_2(M/N) \rceil$ if $N\le M$, Theorem \ref{rep of vector} shows that $\cS_{N,M}$ can be encoded by a deep neural network whose width $W$ and depth $L$ satisfying
\begin{equation}\label{rep upper bound}
W^2L^2 \le C
\begin{cases}
M(1+\log_2(N/M) ), &\mbox{if }N\ge M, \\
N(1+\log_2(M/N) ), &\mbox{if }N\le M.
\end{cases}
\end{equation}
With the above choice of $T$, if the parameter $S$ is chosen properly, one can recover the result of \citet[Theorem 14]{siegel2023optimal}, which corresponds to the case that $W$ is a constant. \citet[Theorem 25]{siegel2023optimal} also proved a lower bound for the size of network when the set $\cS_{N,M}$ can be encoded: there exists a constant $C<\infty$ such that 
\begin{equation}\label{rep lower bound old}
W^4L^2 \ge C^{-1}
\begin{cases}
M(1+\log(N/M) ), &\mbox{if }N\ge M > C\log N, \\
N(1+\log(M/N) ), &\mbox{if }N\le M < \exp(N/C).
\end{cases}
\end{equation}
When the width $W$ is bounded, this lower bound matches the previous upper bound (\ref{rep upper bound}) and hence the construction is optimal. However, when $W$ is not bounded, the lower bound has worse dependence on the width. In the next theorem, we provide a new lower bound, which is better in certain situations.

\begin{theorem}\label{rep lower bound}
Let $\cS_{N,M}$ be the set defined by (\ref{S_NM}). Suppose that $W\ge 2$, $L\ge 1$ and that for any $(x_1,\dots,x_N)^\top\in \cS_{N,M}$, there exists $g\in \NN(W,L)$ such that $g(n) = x_n$ for $n=1,\dots,N$. 
\begin{enumerate}[label=\textnormal{(\arabic*)},parsep=0pt]
\item If $N \ge M \ge C_0 N^{1/p}$ for some $p \ge 1$ and constant $C_0>0$, then 
\[
W^2L^2 \log(WL) \ge C_p M(1+\log(N/M) ),
\]
for some constant $C_p$ depending on $p$ and $C_0$. 
\item If $N \le M \le C_0 N^p$ for some $p \ge 1$ and constant $C_0>0$, then 
\[
W^2L^2 \log(WL) \ge C_p N(1+\log(M/N) ),
\]
for some constant $C_p$ depending on $p$ and $C_0$.
\end{enumerate}
\end{theorem}

This theorem shows that the upper bound (\ref{rep upper bound}) is sharp up to logarithmic factors in the range $cN^{1/p} \le M \le CN^p$. The proof of Theorem \ref{rep lower bound} is given in Appendix \ref{proof of rep lower bound}. Our proof is based on the estimation of the number of sign patterns that can be matched by the neural network which fits $\cS_{N,M}$. The analysis is similar to the derivation of the lower bound (\ref{rep lower bound old}) in \citet[Theorem 25]{siegel2023optimal}. The main difference is that we use the method presented in \citet[Theorem 7]{bartlett2019nearly} to upper bound the number of sign patterns. 

Now, let us come back to the proof of Theorem \ref{rep of vector}. Following the idea of \citet{siegel2023optimal}, we will construct a neural network that implements a pair of encoding and decoding maps $E:\cS_{N,M}\to \{0,1\}^k$ and $D: \{0,1\}^k \to \cS_{N,M}$ which satisfy $D(E(x)) = x$. The encoding and decoding maps are explicitly given by algorithms in \citet{siegel2023optimal}. In the following lemma, we summarize the essential properties of these maps that we need in our construction.

\begin{lemma}\label{encode-decode}
Let $N,M,S \in \bN$ and $x=(x_1,\dots,x_N)^\top\in \bZ^N$ satisfy $\|x\|_1\le M$ and $\|x\|_\infty <S$. We can encode $x$ by a sequence $(f_1,t_1,f_2,t_2,\dots,f_R,t_R)$, where $f_i,t_i\in \bZ$ are further encoded by certain binary sequences described as follows.

\begin{enumerate}[label=\textnormal{(\arabic*)},parsep=0pt]
\item If $N\ge M$, then we have $R\le 2M$, $f_i \in \{0,1,\dots, \lceil N/M\rceil \}$ is encoded via binary expansion with length at most $1+ \lceil \log_2 (N/M) \rceil$ and $t_i\in \{0,\pm 1\}$ is encoded via $0=00$, $1=10$ and $-1=01$. 

\item If $N\le M$, then we have $R\le 2N$, $f_i\in \{0,1\}$ and $t_i\in \{-\lceil M/N \rceil,\dots, \lceil M/N \rceil\}$ is encoded via binary expansion with length at most $2+ \lceil \log_2  (M/N) \rceil$, whose first bit determines its sign and the remaining bits consist of the binary expansion of its magnitude.
\end{enumerate}
Viewing the integers $f_i$ and $t_i$ as the binary sequences above, we define the encoding
\[
E(x) = \Bin 0.f_1t_1 f_2t_2\cdots f_Rt_R,
\]
by using the concatenation of binary sequences and the representation (\ref{binary rep}).

Furthermore, we can decode the above encoding $E(x)$ as follows. Denote $\tau:=S$ if $N\ge M$ and $\tau:=\lceil SN/M \rceil$ if $N\le M$, and let $\rho:=\lceil R/\tau \rceil$. Then there exist two strictly increasing sequences of integers $(i_k)_{k=0}^\rho$ and $(j_k)_{k=0}^\rho$ with $i_0=1$, $i_\rho=R+1$, $j_0=0$, $j_\rho = N$ and $i_k-i_{k-1}<2\tau$ such that
\begin{equation}\label{decoding}
x_n = \sum_{i=i_k}^{i_{k+1}-1} t_i \delta_0\left(n-j_k-\sum_{m=i_k}^i f_m\right),\quad \forall j_k<n\le j_{k+1},
\end{equation}
where $\delta_0(0)=1$ and $\delta_0(z)=0$ for $z\neq 0$.
\end{lemma}
\begin{proof}
The encoding can be implemented by Algorithms 1 and 2 in \citet{siegel2023optimal} respectively. The sequences $(i_k)_{k=0}^\rho$ and $(j_k)_{k=0}^\rho$ are explicitly constructed in the proof of Propositions 15 and 17 in \citet{siegel2023optimal}.
\end{proof}

We illustrate the encoder and decoder in Lemma \ref{encode-decode} by the following two examples.

\begin{example}[Sparse case $N\ge M$] Let $N=5$, $M=4$, $S=3$ and $x=(0,0,-2,0,1)^\top$. Since the vector $x$ is sparse, we use $t_i\in \{0,\pm 1\}$ to encode the non-zero values and use $f_i\in \{0,1,2\}$ to encode their indexes. To do this, we use $(f_1,t_1,f_2) =(2,0,1)$ to encode the index of the first non-zero value $x_3$, where $f_1+f_2=3$ is the index and $t_1=0$ means that we are encoding the index. We encode $x_3=-2$ by $(t_2,f_3,t_3) = (-1,0,-1)$, where $t_2=t_3=\sgn(x_3)$, $|t_2|+|t_3|=|x_3|$ and $f_3=0$ indicates that we do not move to the next index. Finally, we use $f_4=2$ to encode the distance between the current index to the index of next non-zero value $x_5$ and use $t_4=\sgn(x_5)=1$ to encode the value. Thus, the encoding sequence is $(2,0,1,-1,0,-1,2,1)$ and $E(x) = \Bin 0.1000010100011010$.

To decode, we choose $(i_0,i_1,i_2) = (1,2,5)$ and $(j_0,j_1,j_2)=(0,2,5)$. Then, equality (\ref{decoding}) implies that $x_1=x_2=t_1=0$ and for $3\le n\le 5$,
\[
x_n = \sum_{i=2}^{4} t_i \delta_0\left(n-2-\sum_{m=2}^i f_m\right) 
= -\delta_0(n-3) - \delta_0(n-3) + \delta_0(n-5) = 
\begin{cases}
-2 \quad & n=3,\\
0 \quad & n=4,\\
1 \quad & n=5.
\end{cases}
\]
\end{example}

\begin{example}[Dense case $N\le M$] Let $N=3$, $M=7$, $S=5$ and $x= (-4,1,-2)^\top$. Different from the sparse case, we use $t_i\in \{-3,\dots, 3\}$ to encode the values and use $f_i\in \{0,1\}$ to encode the index. We let $f_1=1$ and encode $x_1$ as $(t_1,f_2,t_2) = (-3,0,-1)$, where $t_1+t_2 = x_1$ and $f_2=0$ indicates that we are staying in the same index. We move to the next index by setting $f_3=1$ and encode $x_2$ by $t_3=x_2=1$. Similarly, we move to the next index by setting $f_4=1$ and encode $x_3$ by $t_4=x_3=-2$. Thus, the encoding sequence is $(1,-3,0,-1,1,1,1,-2)$. If we use three bits to encode $t_i$, then $E(x) = \Bin 0.1011000111011010$.

To decode, we choose $(i_0,i_1,i_2) = (1,4,5)$ and $(j_0,j_1,j_2)=(0,2,3)$. Then, equality (\ref{decoding}) implies that, for $n\le 2$,
\[
x_n = \sum_{i=1}^{3} t_i \delta_0\left(n-\sum_{m=1}^i f_m\right) 
= -3\delta_0(n-1) - \delta_0(n-1) + \delta_0(n-2) = 
\begin{cases}
-4 \quad & n=1,\\
1 \quad & n=2.
\end{cases}
\]
For $n=3$, we have $x_3 = t_4 \delta_0\left(3-2- f_4\right) =-2$.
\end{example}

We prove Theorem \ref{rep of vector} by constructing a neural network to implement the map $n \mapsto x_n$ given by the equality (\ref{decoding}).

\begin{proof}[Proof of Theorem \ref{rep of vector}]
Without loss of generality, we can assume that $S\le M$, because the case $S>M$ is a direct consequence of the case $S=M$. We decompose $x\in \bZ^N$ as $x = u+v$, where
\[ 
u_n := 
\begin{cases}
x_n, &|x_n|\ge S,\\
0, &|x_n|< S,
\end{cases}
\qquad v_n := 
\begin{cases}
0, &|x_n|\ge S,\\
x_n, &|x_n|< S.
\end{cases}
\]
Notice that the large part $u$ has small support:
\[
|\{n:u_n\neq 0\}| \le \frac{\|u\|_1}{S} \le \frac{\|x\|_1}{S} \le \frac{M}{S}.
\]
Thus, there exists a continuous piecewise linear function with at most $3M/S$ pieces which matches the values of $u$. By Lemma \ref{PWL rep}, for any $W_1,L_1\in \bN$ (which will be chosen later) satisfying $2W_1^2L_1 \ge M/S$, there exists $g_u\in \NN(6W_1+2,2L_1)$ such that $g_u(n)=u_n$ for $n=1,\dots,N$. 

It remains to construct a neural network $g_v$ to represent the small part by applying Lemma \ref{encode-decode} to $v\in \bZ^N$. Notice that, by equality (\ref{decoding}), we only need to know $n-j_k$ and $r_k= \Bin 0.f_{i_k} t_{i_k}\cdots f_{i_{k+1}-1} t_{i_{k+1}-1}$ to compute $v_n$. We define two continuous piecewise linear functions on $[0,N]$ by
\begin{align*}
J(z) &= 
\begin{cases}
z-j_k, \quad &j_k+1\le z\le j_{k+1},\\
\mbox{linear}, \quad &j_k<z < j_k+1,
\end{cases} \\
R(z) &= 
\begin{cases}
r_k, \quad &j_k+1\le z\le j_{k+1},\\
\mbox{linear}, \quad &j_k<z < j_k+1.
\end{cases}
\end{align*}
Then, $J(n)=n-j_k$ and $R(n)=r_k$ for $j_k<n\le j_{k+1}$. Observe that $J$ and $R$ has at most $2\rho$ pieces with
\[
\rho \le
\begin{cases}
\left\lceil \frac{2M}{S} \right\rceil \le \frac{3M}{S}, &\mbox{if }N\ge M, \\
\left\lceil \frac{2N}{\lceil SN/M \rceil} \right\rceil \le \left\lceil \frac{2N}{SN/M } \right\rceil = \left\lceil \frac{2M}{S} \right\rceil \le \frac{3M}{S}, &\mbox{if }N< M.
\end{cases}
\]
By Lemma \ref{PWL rep}, if $W_1^2L_1 \ge M/S$, then $J,R \in \NN(6W_1+2,2L_1)$ and
\begin{equation}\label{temp nn1}
n \to 
\begin{pmatrix}
n-j_k \\
r_k \\
0
\end{pmatrix}
\in \NN(12W_1+4,2L_1), \quad \mbox{for }j_k<n\le j_{k+1}.
\end{equation}
Next, we use Lemma \ref{bit extraction} to extract $f_i$ and $t_i$, $i_k\le i<i_{k+1}$, from $r_k$ and compute $v_n$ using (\ref{decoding}).

We first consider the case $N\ge M$. Recall that $i_{k+1}-i_k<2\tau$. We are going to construct a neural network to implement the following map
\begin{equation}\label{temp nn2}
\begin{pmatrix}
z \\
\Bin 0.f_1t_1\cdots f_{2\tau}t_{2\tau} \\
\Sigma
\end{pmatrix}
\to 
\begin{pmatrix}
z-f_1 \\
\Bin 0.f_2t_2\cdots f_{2\tau}t_{2\tau} \\
\Sigma + t_1 \delta_0(z-f_1)
\end{pmatrix},
\end{equation}
where $z,\Sigma\in \bZ$, $f_i\in \bN$ is encoded via binary representation with length $\alpha :=1+ \lceil \log_2 (N/M) \rceil$ and $t_i\in \{0,\pm 1\}$ is encoded via $0=00$, $1=10$ and $-1=01$. The construction is similar to \citet[Lemma 16]{siegel2023optimal}, but we pay more attention to the trade-off between width and depth. We first apply the network $f_{2\tau(\alpha+2),\alpha,T}\in \NN(2^{\lceil\alpha/T \rceil+2}+2,T)$ in Lemma \ref{bit extraction} to extract $f_1$ from the second component. Since $2^{\lceil\alpha/T \rceil+2} \le 16(N/M)^{1/T}$, we have
\[
\begin{pmatrix}
z \\
\Bin 0.f_1t_1\cdots f_{2\tau}t_{2\tau} \\
\Sigma
\end{pmatrix}
\to 
\begin{pmatrix}
z - f_1 \\
\Bin 0.t_1f_2t_2\cdots f_{2\tau}t_{2\tau} \\
\Sigma
\end{pmatrix}
\in \NN(16\lceil (N/M)^{1/T} \rceil+6,T).
\]
Notice that the delta function $\delta_0$ on $\bZ$ can be implemented by the following piecewise linear function
\begin{equation}\label{app delta}
h(z) = 
\begin{cases}
0, & |z|\ge 1,\\
1+z, & -1<z\le 0, \\
1-z, & 0<z<1,
\end{cases}
\end{equation}
which is in $\NN(3,1)$ by Lemma \ref{PWL rep}. We apply $h$ to the first component and use the network $f_{2\tau(\alpha+2),2} \in \NN(17,1)$ in Lemma \ref{bit extraction} to extract the two bits $b_1,b_2$ corresponding to $t_1$ from the second component. This gives
\[
\begin{pmatrix}
z - f_1 \\
\Bin 0.t_1f_2t_2\cdots f_{2\tau}t_{2\tau} \\
\Sigma
\end{pmatrix}
\to
\begin{pmatrix}
z - f_1 \\
h(z-f_1) \\
b_1 \\
b_2 \\
\Bin 0.f_2t_2\cdots f_{2\tau}t_{2\tau} \\
\Sigma
\end{pmatrix}
\in \NN(24,1).
\]
It remains to implement the map $(a,b_1,b_2) \to at_1$, where $a=\delta_0(z-f_1) = h(z-f_1) \in \{0,1\}$, by a network. This can be done by the observation that $t_1 = b_1-b_2$ and 
\[
at_1 = ab_1 -ab_2 = \sigma(a+b_1-1) - \sigma(a+b_2-1).
\]
Combining the above constructions, we obtain that the map (\ref{temp nn2}) is in $\NN(16\lceil (N/M)^{1/T} \rceil+8,T+2)$ by Proposition \ref{basic constr}. 

We compose the network in (\ref{temp nn1}) with $2\tau$ copies of the network in (\ref{temp nn2}) and then use an affine map to select the last component. By equality (\ref{decoding}), this gives us a network $g_v\in \NN(16\max\{W_1, \lceil (N/M)^{1/T} \rceil\}+8,2L_1+2\tau(T+2))$ satisfying $g_v(n) = v_n$ for $n=1,\dots,N$. Note that we pad $r_k= \Bin 0.f_{i_k} t_{i_k}\cdots f_{i_{k+1}-1} t_{i_{k+1}-1}$ with zeros so that it has $2\tau$ blocks of $(f_i,t_i)$. The additional zero blocks have no effect on the computation of $v_n$. As a consequence, we get the desired network $g=g_u+g_v \in \NN(22\max\{W_1, \lceil (N/M)^{1/T} \rceil\}+10,2L_1+2\tau(T+2))$. Now, we choose
\[
L_1=\tau(T+2) = S(T+2).
\]
Recall the requirement that $W_1^2L_1 \ge M/S$, which implies we can choose
\[
W_1 = \left\lceil \frac{\sqrt{M}}{S \sqrt{T+2}} \right\rceil.
\]

If $N\le M$, we need to construct a network to implement the map (\ref{temp nn2}), where $z,\Sigma\in \bZ$, $f_i\in \{0,1\}$ and $t_i\in \{-\lceil M/N \rceil,\dots, \lceil M/N \rceil\}$ is encoded via binary expansion with length $\beta+1 := \lceil \log_2  (M/N) \rceil +2$, whose first bit determines its sign and the remaining bits consist of the binary expansion of its magnitude. For convenience, let us denote the bits of $t_1$ by $b_0b_1\cdots b_\beta$, where $b_0=0$ if $t_1< 0$ and $b_0=1$ otherwise. We first apply the network $f_{2\tau(\beta+2),2} \in \NN(17,1)$ in Lemma \ref{bit extraction} to extract $f_1$ and $b_0$ from the second component
\[
\begin{pmatrix}
z \\
\Bin 0.f_1t_1\cdots f_{2\tau}t_{2\tau} \\
\Sigma
\end{pmatrix}
\to 
\begin{pmatrix}
z - f_1 \\
b_0 \\
\Bin 0.b_1\cdots b_\beta f_2t_2\cdots f_{2\tau}t_{2\tau} \\
\Sigma
\end{pmatrix}
\in \NN(21,1).
\]
Then, we can approximate the delta function by $h\in \NN(3,1)$ defined as (\ref{app delta}) and use the network $f_{2\tau(\beta+2),\beta,T}\in \NN(2^{\lceil \beta/T\rceil +2}+2,T)$ in Lemma \ref{bit extraction} to compute $|t_1|$ from $b_1\cdots b_\beta$. Specifically, 
\[
\begin{pmatrix}
z - f_1 \\
b_0 \\
\Bin 0.b_1\cdots b_\beta f_2t_2\cdots f_{2\tau}t_{2\tau} \\
\Sigma
\end{pmatrix}
\to
\begin{pmatrix}
z - f_1 \\
h(z-f_1) \\
b_0 \\
|t_1| \\
\Bin 0.f_2t_2\cdots f_{2\tau}t_{2\tau} \\
\Sigma
\end{pmatrix}
\in \NN(16\lceil (M/N)^{1/T} \rceil+10,T),
\]
because $2^{\lceil \beta/T\rceil +2} \le 16(M/N)^{1/T}$. It remains to implement the map $(a,b_0,|t_1|) \to at_1$, where $a=h(z-f_1) \in \{0,1\}$, by a network. Observing that
\[
at_1 = 
\begin{cases}
0, &a=0, \\
|t_1|, &a=1,b_0=1, \\
-|t_1|, &a=1,b_0=0,
\end{cases}
\]
and $|t_1|<2^\beta$, we have
\[
at_1 = \sigma(|t_1|-2^\beta(2-a-b_0)) - \sigma(|t_1|-2^\beta(1-a+b_0)).
\]
Hence, the map (\ref{temp nn2}) is in $\NN(16\lceil (M/N)^{1/T} \rceil+10,T+2)$ by Proposition \ref{basic constr}.

As before, we compose the network in (\ref{temp nn1}) with $2\tau$ copies of the network in (\ref{temp nn2}) to get $g_v \in \NN(16\max \{W_1,\lceil (M/N)^{1/T} \rceil\}+10, 2L_1+2\tau(T+2))$ satisfying $g_v(n) = v_n$ for $n=1,\dots,N$. Consequently, we get the desired network $g=g_u+g_v \in \NN(22\max \{W_1,\lceil (M/N)^{1/T} \rceil\}+12,2L_1+2\tau(T+2))$. By choosing
\[
L_1=\tau(T+2) = \lceil SN/M \rceil(T+2),
\]
and
\[
W_1 = \left\lceil \frac{M}{S \sqrt{N(T+2)}} \right\rceil,
\]
we fulfill the requirement $W_1^2L_1 \ge M/S$ and complete the proof.
\end{proof}

\subsection{Approximation of piecewise polynomials}\label{sec: app of poly}

Since the seminal work of \citet{yarotsky2017error}, it is well-known that polynomials can be efficiently approximated by deep ReLU neural networks. The starting point of this theory is the approximation of the product function $(x,y) \mapsto xy$. The following lemma is a modification of \citet[Lemma 5.1]{lu2021deep}.

\begin{lemma}\label{app of product}
For any integers $k\ge 4$ and $L\ge 1$, there exists $f_{k,L}\in \NN(3k2^k+3,L)$ such that $f_{k,L}: [-1,1]^2 \to [-1,1]$ and
\[
|f_{k,L}(x,y) - xy| \le 2^{-2kL-1}, \quad \forall x,y\in [-1,1].
\]
\end{lemma}

Once we have the approximation of the product function, we can approximate any monomials by viewing them as multi-products and hence can approximate any polynomials. This idea can be further combined with Lemma \ref{app partition} and Theorem \ref{rep of vector} to approximate piecewise polynomials on the good region $\Omega_{\ell,\epsilon}$ defined by (\ref{good region}). We prepare the following proposition for the purpose of proving Theorems \ref{app of sobolev} and \ref{app of besov}.

\begin{proposition}\label{app of poly}
Let $\ell,k \in \bN_0$ and $0<\epsilon<b^{-\ell}$. Suppose that $f\in \cP_k^\ell$ is expanded in terms of the
bases $\rho_{\ell,\Bi}^\gamma$ defined by (\ref{basis}),
\[
f(x) = \sum_{|\gamma|\le k, \Bi \in I_\ell} a_{\Bi,\gamma} \rho_{\ell,\Bi}^\gamma(x).
\]
Let $1\le q\le p\le \infty$ and choose a parameter $\delta > 0$. The following approximation results hold for some constants $C:=C(p,q,d,k,b)$ depending only on $p,q,d,k$ and the base $b$. 
\begin{enumerate}[label=\textnormal{(\arabic*)},parsep=0pt]
\item If $\delta q \le d\ell$, then for any $S,T,W_0,L_0 \in \bN$, there exists $g\in \NN(W,L)$ with 
\begin{align*}
W &\le C \max \left\{ \frac{b^{\delta q/2}}{S \sqrt{T}}, b^{(d\ell -\delta q)/T}, b^{\ell/L_0 }, W_0 2^{W_0} \right\}, \\
L &\le C(ST+L_0),
\end{align*}
such that (with the standard modification when $q=\infty$)
\[
\|f-g\|_{L^p(\Omega_{\ell,\epsilon})} \le C \left(b^{\delta q/p -d\ell/p-\delta} + 4^{-W_0L_0}\right) \left( \sum_{|\gamma|\le k,\Bi \in I_\ell} |a_{\Bi,\gamma}|^q \right)^{1/q}.
\]
\item If $\delta q \ge d\ell$, then for any $S,T,W_0,L_0 \in \bN$, there exists $g\in \NN(W,L)$ with 
\begin{align*}
W &\le C \max \left\{ \frac{b^{\delta+d\ell/2-d\ell/q}}{S \sqrt{T}} , b^{(\delta-d\ell/q)/T}, b^{\ell/L_0 }, W_0 2^{W_0} \right\}, \\
L &\le C \left(\left\lceil b^{-\delta + d\ell/q} S \right\rceil T+L_0 \right),
\end{align*}
such that (with the standard modification when $q=\infty$)
\[
\|f-g\|_{L^p(\Omega_{\ell,\epsilon})} \le C \left(b^{-\delta} + 4^{-W_0L_0}\right) \left( \sum_{|\gamma|\le k,\Bi \in I_\ell} |a_{\Bi,\gamma}|^q \right)^{1/q}.
\]
\end{enumerate}
\end{proposition}

This proposition may seem to be complicated at first glance. So let us explain the intuition and meaning of the parameters $\delta,S,T,W_0,L_0$. The approximation of the piecewise polynomial can be divided into two parts. The first part is the approximation of the coefficients $a_{\Bi,\gamma}$. We will first discretize these coefficients and then encode them by using the network from Theorem \ref{rep of vector}, which gives us two tunable parameters $S$ and $T$. The parameter $\delta$ (more precisely $b^{-\delta}$) represents the discretization level. The conditions $\delta q \le d\ell$ and $\delta q \ge d\ell$ correspond to the sparse and dense regimes respectively in Theorem \ref{rep of vector}. The second part is the approximation of the base functions $\rho_{\ell,\Bi}^\gamma(x)$ and the product $(a_{\Bi,\gamma}, \rho_{\ell,\Bi}^\gamma(x)) \mapsto a_{\Bi,\gamma} \rho_{\ell,\Bi}^\gamma(x)$. This can be done by using lemmas \ref{app partition} and \ref{app of product}. The parameters $W_0$ and $L_0$ are used to tune the size of networks constructed in the second part. We remark that, in general, the networks in the second part can be much smaller than the encoding network in the first part, because their approximation errors decay as $4^{-W_0L_0}$.

\begin{proof}[Proof of Proposition \ref{app of poly}]
Let us consider the decomposition $f = \sum_{|\gamma|\le k} f_\gamma$ where
\[
f_\gamma(x) = \sum_{\Bi \in I_\ell} a_{\Bi,\gamma} \rho_{\ell,\Bi}^\gamma(x).
\]
By Proposition \ref{basic constr} and the triangle inequality, it is enough to prove the result for each $f_\gamma$ at the expense of larger constants. Thus, we assume that $f:= f_\gamma$ and write $a_{\Bi} :=a_{\Bi,\gamma}$ in the following analysis. Without loss of generality, we can further assume that 
\[
\|a \|_q := \left( \sum_{\Bi \in I_\ell} |a_{\Bi}|^q \right)^{1/q} \le 1
\]
with the standard modification when $q=\infty$, where $a:=(a_\Bi)_{\Bi \in I_\ell}$ denotes the vector of coefficients. In order to use Theorem \ref{rep of vector}, we will need to discretize these coefficients. Given $\delta >0$, we can approximate $a$ by $\widetilde{a}=(\widetilde{a}_\Bi)_{\Bi \in I_\ell}$ with
\[
\widetilde{a}_\Bi:= b^{-\delta} \sgn(a_\Bi) \lfloor b^\delta |a_\Bi| \rfloor.
\]
It is easy to see that $\|a-\widetilde{a}\|_\infty \le b^{-\delta}$ and $\|a-\widetilde{a}\|_q \le \|a\|_q \le 1$. Since $|I_\ell|=b^{d\ell}$, the uniform bound implies $\|a-\widetilde{a}\|_p \le b^{d\ell/p -\delta}$. On the other hand, since $p\ge q$,
\[
\|a-\widetilde{a}\|_p \le \|a-\widetilde{a}\|_q^{q/p} \|a-\widetilde{a}\|_\infty^{1-q/p} \le b^{\delta q/p -\delta}.
\]
In summary, we have
\begin{equation}\label{bound for coe}
\|a-\widetilde{a}\|_p \le b^{-\delta} \min\{ b^{d\ell/p}, b^{\delta q/p} \}.
\end{equation}

To construct the desired network $g$, we first apply $q_1,q_d \in \NN(2db^{\lceil \ell/L_0 \rceil},L_0)$ in Lemma \ref{app partition} to compute the index of the input
\begin{equation}\label{compute ind}
x \to 
\begin{pmatrix}
q_1(x_1) \\
\vdots \\
q_1(x_d) \\
q_d(x) \\
x
\end{pmatrix}
\to 
\begin{pmatrix}
b^\ell x_1 - q_1(x_1) \\
\vdots \\
b^\ell x_d - q_1(x_d) \\
q_d(x) 
\end{pmatrix}.
\end{equation}
This map can be implemented by a network with width $W_1 \le C b^{\ell/L_0 }$ and depth $L_1 = L_0$. Furthermore, for $x\in \Omega_{\Bi,\epsilon}^\ell$, this map becomes $x \to (b^\ell x_1 - \Bi_1,\cdots, b^\ell x_d - \Bi_d,\ind(\Bi))^\top$. Next, we construct a neural network to implement the map $\ind(\Bi) \mapsto b^{-\delta}u_{\ind(\Bi)}$, where we define $u\in \bZ^{b^{d\ell}}$ as the vector whose $\ind(\Bi)$-th entry is given by $u_{\ind(\Bi)} = b^{\delta} \widetilde{a}_\Bi = \sgn(a_\Bi) \lfloor b^\delta |a_\Bi| \rfloor$. This can be done by using Theorem \ref{rep of vector}. We let $N=b^{d\ell}$ and estimate $\|u\|_1$ as follows. Observe that $\|u\|_q \le b^\delta \|a\|_q \le b^\delta$, which implies
\[
|\{i:u_i\neq 0 \}| \le \min\{ b^{\delta q},N \},
\]
since $u\in \bZ^N$. Using H\"older's inequality, we get
\[
\|u\|_1 \le |\{i:u_i\neq 0 \}|^{1-1/q} \|u\|_q \le b^\delta \min\{ b^{\delta q},b^{d\ell} \}^{1-1/q}.
\]
Thus, we can apply Theorem \ref{rep of vector} with $M=b^\delta \min\{ b^{\delta q},b^{d\ell} \}^{1-1/q}$ to the vector $u$ and get a network $\phi \in \NN(W_2,L_2)$ such that $\phi(\ind(\Bi)) = \widetilde{a}_\Bi$. If $\delta q \le d\ell$, then $M=b^{\delta q} \le b^{d\ell} = N$ and we can choose 
\[
W_2\le C \max \left\{ \left\lceil \frac{b^{\delta q/2}}{S \sqrt{T}} \right\rceil , b^{(d\ell -\delta q)/T} \right\}, \quad L_2\le CST.
\]
If $\delta q \ge d\ell$, then $M = b^{\delta +d\ell - d\ell/q} \ge b^{d\ell} =N$ and we can choose
\[
W_2 \le C \max \left\{ \left\lceil \frac{b^{\delta+d\ell/2-d\ell/q}}{S \sqrt{T}} \right\rceil, b^{(\delta-d\ell/q)/T} \right\}, \quad L_2 \le C \left\lceil b^{-\delta + d\ell/q} S \right\rceil T.
\]
Composing $\phi$ with the last component of the network (\ref{compute ind}), we get a neural network $h\in \NN(\max\{W_1,W_2+2d\},L_1+L_2)$, which satisfies
\[
h(x) =
\begin{pmatrix}
b^\ell x_1 - \Bi_1 \\
\vdots \\
b^\ell x_d - \Bi_d \\
\widetilde{a}_\Bi
\end{pmatrix}
\in [-1,1]^{d+1}, \quad \forall x \in \Omega_{\Bi,\epsilon}^\ell.
\]
The final step is to approximate the polynomial $(y,z) \mapsto z \prod_{j=1}^d y_j^{\gamma_j}$ for $y\in[-1,1]^d$ and $z\in [-1,1]$. By using the network $f_{W_0+3,L_0} \in \NN(24(W_0+3)2^{W_0}+3,L_0)$ from Lemma \ref{app of product}, we define the following neural networks
\[
P_j:
\begin{pmatrix}
y \\
z
\end{pmatrix}
\to
\begin{pmatrix}
y \\
f_{W_0+3,L_0}(y_j,z)
\end{pmatrix},
\quad j=1,\dots,d.
\]
We construct $P_\gamma$ by composing $\gamma_j$ copies of $P_j$ and then applying an affine map which selects the last coordinate. Consequently, $P_\gamma \in \NN(W_3,L_3)$ with $W_3 \le CW_0 2^{W_0}$ and $L_3 \le kL_0$. Since all entries are in $[-1,1]$, Lemma \ref{app of product} implies that the approximation error can be bounded as
\[
\left|P_\gamma(y,z) - z \prod_{j=1}^d y_j^{\gamma_j} \right| \le \sum_{j=1}^d \gamma_j 2^{-2(W_0+3)L_0-1} \le C 4^{-W_0L_0}.
\]
By composing $P_\gamma$ with $h$, we obtain the desired network $g\in \NN(W,L)$, whose width $W\le C \max\{W_1,W_2,W_3\}$ and depth $L= L_1+L_2+L_3$, such that $g(x) = P_\gamma(b^\ell x-\Bi,\widetilde{a}_\Bi)$ for $x \in \Omega_{\Bi,\epsilon}^\ell$.

It remains to estimate the approximation error of $g$. Using the fact that the basis $\rho_{\ell,\Bi}^\gamma$ has disjoint support, we have (with obvious modification for $p=\infty$)
\begin{align*}
\|f-g\|_{L^p(\Omega_{\ell,\epsilon})}^p &= \sum_{\Bi \in I_\ell} \int_{\Omega_{\Bi,\epsilon}^\ell} \left| a_\Bi \prod_{j=1}^d (b^{\ell}x_j-\Bi_j)^{\gamma_j} - P_\gamma(b^\ell x-\Bi,\widetilde{a}_\Bi) \right|^p dx \\
&\le 2^{p-1} \sum_{\Bi \in I_\ell} \int_{\Omega_{\Bi,\epsilon}^\ell} |a_\Bi - \widetilde{a}_\Bi|^p + \left| \widetilde{a}_\Bi \prod_{j=1}^d (b^{\ell}x_j-\Bi_j)^{\gamma_j} - P_\gamma(b^\ell x-\Bi,\widetilde{a}_\Bi) \right|^p dx \\
&\le Cb^{-d\ell} \|a-\widetilde{a}\|_p^p + C4^{-pW_0L_0}.
\end{align*}
By inequality (\ref{bound for coe}), 
\[
\|f-g\|_{L^p(\Omega_{\ell,\epsilon})} \le C \left(b^{-\delta} \min\{ 1, b^{\delta q/p-d\ell/p} \} + 4^{-W_0L_0}\right),
\]
which completes the proof.
\end{proof}

\subsection{Approximation of Sobolev functions}\label{sec: app of sobolev}

In this subsection, we will use Proposition \ref{app of poly} to derive bounds for the neural network approximation of Sobolev functions and give a proof of Theorem \ref{app of sobolev}. We remark that almost the same argument can be applied to obtain approximation bounds for Besov functions and prove Theorem \ref{app of besov}. The main differences are given in Remark \ref{besov detail}.

\begin{proposition}\label{app of sobolev part}
Let $1\le q\le p\le \infty$, $s>0$ and $f\in \cW^{s,q}([0,1]^d)$ with $\|f\|_{\cW^{s,q}([0,1]^d)} \le 1$. Assume that the Sobolev embedding condition is strictly satisfied, i.e. $1/q - 1/p <s/d$. Let $\alpha,\beta \in \bN$ and $0< \epsilon< b^{-\ell^*}$, where $\ell^* = \lfloor 2\kappa (\alpha+\beta) \rfloor$ with
\[
\kappa := \frac{s}{s+d/p-d/q} \in [1,\infty).
\]
Then, there exists a network $g_{\alpha,\beta} \in \NN(W,L)$ with $W \le Cb^{d\alpha}$ and $L\le Cb^{d\beta}$ such that
\[
\|f-g_{\alpha,\beta} \|_{L^p(\Omega_{\ell^*,\epsilon})} \le C b^{-2s(\alpha+\beta)}.
\]
Here the constants $C:=C(p,q,s,d,b)$ depend only on $p,q,s,d$ and the base $b$.
\end{proposition}

\begin{proof}
Let us consider the $L^q$-projection of $f\in L^q(\Omega)$ onto the space of piecewise polynomials of degree $k=\lfloor s \rfloor$ defined as
\[
\Pi_k^\ell(f) := \argmin_{h\in \cP_k^\ell} \|f-h\|_{L^q(\Omega)}.
\]
Let $f_0 = \Pi_k^0(f)$ and $f_\ell = \Pi_k^\ell(f) - \Pi_k^{\ell-1}(f)$ for $\ell\in \bN$. Then, we have the following decomposition
\[
f = \sum_{\ell=0}^\infty f_\ell.
\]
By expanding $f_\ell$ in the basis $\rho_{\ell,\Bi}^\gamma$, we can write
\[
f_\ell(x) = \sum_{|\gamma|\le k, \Bi \in I_\ell} a_{\Bi,\gamma}(\ell) \rho_{\ell,\Bi}^\gamma(x).
\]
By using the Bramble-Hilbert lemma, namely $\|\Pi_k^0(f) - f\|_{L^q(\Omega)} \le C|f|_{\cW^{s,q}(\Omega)}$, and a scaling argument (see \citet[Eq. (4.12)]{siegel2023optimal} for details), one can show that the coefficients satisfy the following bound
\[
|a_{\Bi,\gamma}(\ell)| \le C b^{(d/q-s)\ell} |f|_{\cW^{s,q}(\Omega_{\Bi^-}^{\ell-1})},
\]
where $\Omega_{\Bi^-}^{\ell-1} \supset \Omega_\Bi^\ell$ is the parent domain of $\Omega_\Bi^\ell$ for $\ell \ge 1$. When $\ell =0$, we have the modification $|a_{\boldsymbol{0},\gamma}(0)|\le C\|f\|_{\cW^{s,q}(\Omega)}$. Combining this bound with the sub-additivity of the Sobolev norm
\begin{equation}\label{sub-add}
\sum_{\Bi \in I_\ell} |f|_{\cW^{s,q}(\Omega_\Bi^\ell)}^q \le |f|_{\cW^{s,q}(\Omega)}^q,
\end{equation}
we get a bound for the $\ell^q$-norm of the coefficients (with the standard modification for $q=\infty$)
\begin{equation}\label{norm bound for coe}
\left( \sum_{|\gamma|\le k,\Bi \in I_\ell} |a_{\Bi,\gamma}(\ell)|^q \right)^{1/q} \le C b^{(d/q-s)\ell} \left( \sum_{|\gamma|\le k,\Bi \in I_\ell} |f|_{\cW^{s,q}(\Omega_{\Bi^-}^{\ell-1})}^q \right)^{1/q} \le C b^{(d/q-s)\ell},
\end{equation}
because $\|f\|_{\cW^{s,q}(\Omega)} \le 1$ and each $\Omega_{\Bi^-}^{\ell-1}$ appears $\binom{k+d}{d}b^d$ times in the summation. As a consequence, by using the fact that the basis $\rho_{\ell,\Bi}^\gamma$ has disjoint support,
\begin{equation}\label{norm bound for fl}
\|f_\ell\|_{L^p(\Omega)} \le C b^{-d\ell/p} \left( \sum_{|\gamma|\le k,\Bi \in I_\ell} |a_{\Bi,\gamma}(\ell)|^p \right)^{1/p} \le C b^{(d/q-d/p-s)\ell}.
\end{equation}

Next, we are going to construct a neural network $g_\ell$ to approximate $f_\ell$ for each $\ell \in \cL := \{0,\dots, \ell^*\}$ by applying Proposition \ref{app of poly}. The key of our proof is to choose the parameters $\delta,S,T,W_0,L_0$ in Proposition \ref{app of poly} appropriately as functions of $\ell$. For each $\ell \in \cL$, we choose
\begin{align*}
W_0(\ell) &= \left\lceil (\alpha/2) \log_2 b \right\rceil,  \\
L_0(\ell) &=  \left\lceil 4(s+\kappa+\kappa d/q) \beta \right\rceil,
\end{align*}
which imply that (since $\ell \le \ell^* \le 2\kappa(\alpha+\beta)$)
\begin{align*}
b^{\ell/L_0 } &\le b^{2\kappa(\alpha+\beta)/(4\kappa \beta)} \le C b^{\alpha/2}, \\
W_02^{W_0} &\le C \alpha b^{\alpha/2}.
\end{align*}
Using the inequality $\alpha \beta \ge (\alpha+\beta)/2$ for $\alpha,\beta \ge 1$, we get
\begin{equation}\label{small error}
4^{-W_0L_0} \le b^{-\alpha L_0} \le b^{-4(s+\kappa d/q)\alpha \beta} \le b^{-2s(\alpha+\beta) - d\ell^*/q}.
\end{equation}
In order to choose the remaining parameters $S,T$ and $\delta$,  we will need to decompose the index set $\cL$ into two groups according to the two cases in Proposition \ref{app of poly}. Besides, when we compute the summation of these small networks in each group, we will need to further decompose each group into two sets, so that we can control the size of the entire network by applying Proposition \ref{basic constr} Part (4) in two different ways (see the discussion below Proposition \ref{basic constr} for an explanation). Hence, we decompose the index set $\cL$ into four disjoint sets $\cL = \cup_{i=1}^4 \cL_i$, which will be given explicitly below. For each $\cL_i$, we denote
\[
F_i:= \sum_{\ell \in \cL_i} f_\ell,\quad \mbox{and}\quad G_i:= \sum_{\ell \in \cL_i} g_\ell.
\]
By choosing $S(\ell), T(\ell)$ and $\delta(\ell)$ appropriately, we are going to derive a bound for the approximation error $\|f_\ell-g_\ell\|_{L^p(\Omega_{\ell,\epsilon})} $ and show that $G_i \in \NN(W_i,L_i)$ with $W_i\le Cb^{d\alpha}$ and $L_i\le Cb^{d\beta}$ for each $i=1,2,3,4$. Thus, we divide the analysis into four cases. 

\textbf{Case 1:} $\ell \in \cL_1=\{\ell \in \cL: 0\le \ell \le 2\beta\}$. We choose
\[
\delta(\ell) = d\ell/q + (s+1)(2\alpha+2\beta -\ell).
\]
Since $\delta q \ge d\ell$, we apply Proposition \ref{app of poly} Part (2) to approximate $f_\ell$ with parameters
\begin{align*}
S(\ell) &= \left\lceil b^{\delta+d\ell/2-d\ell/q} \right\rceil, \\
T(\ell) &= \left\lceil (s+1)(2+2\beta-\ell)/d \right\rceil.
\end{align*}
This gives us a neural network $g_\ell \in \NN(W(\ell),L(\ell))$ with width
\begin{align*}
W(\ell) &\le C \max \left\{ \frac{b^{\delta+d\ell/2-d\ell/q}}{S \sqrt{T}}, b^{(\delta-d\ell/q)/T}, b^{\ell/L_0 }, W_0 2^{W_0} \right\} \\
&\le C \max \left\{ 1, b^{d(\alpha+1)}, b^{\alpha/2}, \alpha b^{\alpha/2} \right\} \\
&\le Cb^{d\alpha},
\end{align*}
and depth
\[
L(\ell) \le C \left(\left\lceil b^{-\delta + d\ell/q} S \right\rceil T+L_0 \right) \le C \left(b^{d\ell/2}(1+2\beta-\ell) + \beta \right).
\]
Moreover, by Proposition \ref{app of poly} and inequality (\ref{norm bound for coe}), the approximation error satisfies
\begin{equation}\label{error 1}
\begin{aligned}
\|f_\ell-g_\ell\|_{L^p(\Omega_{\ell,\epsilon})} &\le C \left(b^{-\delta} + 4^{-W_0L_0}\right) b^{d\ell/q-s\ell} \\
&\le C b^{-2s(\alpha+\beta)} \left( b^{-2\alpha-2\beta+\ell} +b^{-s\ell} \right),
\end{aligned}
\end{equation}
where we use (\ref{small error}) in the last inequality. By Proposition \ref{basic constr}, we can construct the sum $G_1 = \sum_{\ell\in\cL_1} g_\ell$ in a sequential way, so that $G_1\in \NN(W_1,L_1)$ with
\begin{align*}
W_1 &= \max_{\ell \in \cL_1} W(\ell)+2d+2 \le Cb^{d\alpha}, \\
L_1 &= \sum_{\ell \in \cL_1} L(\ell) \le C \sum_{\ell=0}^{2\beta} \left(b^{d\ell/2}(1+2\beta-\ell) + \beta \right) \le C b^{d\beta},
\end{align*}
because the sum is bounded by convergent geometric series.

\textbf{Case 2:} $\ell \in \cL_2=\{\ell \in \cL: 2\beta < \ell \le 2\alpha + 2\beta\}$. We choose $\delta(\ell)$ as in Case 1 so that $\delta q \ge d\ell$ and we can apply Proposition \ref{app of poly} Part (2) again. But we set the parameters
\begin{align*}
S(\ell) &= \left\lceil b^{\delta+d\beta-d\ell/q} \right\rceil, \\
T(\ell) &= \left\lceil 2(s+1)/d \right\rceil.
\end{align*}
This gives us a neural network $g_\ell \in \NN(W(\ell),L(\ell))$ with width
\begin{align*}
W(\ell) &\le C \max \left\{ \frac{b^{\delta+d\ell/2-d\ell/q}}{S \sqrt{T}}, b^{(\delta-d\ell/q)/T}, b^{\ell/L_0 }, W_0 2^{W_0} \right\} \\
&\le C \max \left\{ b^{d\ell/2-d\beta}, b^{d(\alpha+\beta-\ell/2)}, b^{\alpha/2}, \alpha b^{\alpha/2} \right\} \\
&\le Cb^{d\alpha} \left( b^{d(\ell/2-\alpha-\beta)} + b^{d(\beta - \ell/2)} + \alpha b^{-\alpha/2} \right),
\end{align*}
and depth
\[
L(\ell) \le C \left(\left\lceil b^{-\delta + d\ell/q} S \right\rceil T+L_0 \right) \le C \left(b^{d\beta} + \beta \right) \le Cb^{d\beta}.
\]
By Proposition \ref{app of poly} and inequality (\ref{norm bound for coe}), we also get the approximation error (\ref{error 1}) in this case. Using Proposition \ref{basic constr}, we can construct the sum $G_2 = \sum_{\ell\in\cL_2} g_\ell$ in a parallel way, so that $G_2\in \NN(W_2,L_2)$ with
\begin{align*}
W_2 &= \sum_{\ell \in \cL_2} W(\ell) \le Cb^{d\alpha} \sum_{\ell=2\beta+1}^{2\alpha+2\beta} \left( b^{d(\ell/2-\alpha-\beta)} + b^{d(\beta - \ell/2)} + \alpha b^{-\alpha/2} \right) \le Cb^{d\alpha}, \\
L_2 &= \max_{\ell \in \cL_2} L(\ell) \le  C b^{d\beta},
\end{align*}
since the first two series are convergent geometric series.

\textbf{Case 3:} $\ell \in \cL_3=\{\ell \in \cL: 2\alpha+2\beta < \ell \le 2\alpha+2\beta + 2d\alpha/\tau\}$, where $\tau$ is chosen to satisfy 
\begin{equation}\label{tau}
0< \tau < \frac{s}{1/q-1/p} -d = \frac{d}{\kappa -1}.
\end{equation}
Note that this condition can be satisfied since $q\le p$ and the Sobolev embedding condition holds. In this case, we let
\[
\delta(\ell) = 2d(\alpha+\beta)/q - \tau (\ell - 2\alpha- 2\beta)/q.
\]
Since $\delta q < d(2\alpha+2\beta) <d\ell$, we can apply Proposition \ref{app of poly} Part (1) to approximate $f_\ell$ with parameters
\begin{align*}
S(\ell) &= b^{d\beta}, \\
T(\ell) &= \left\lceil 2d/\tau \right\rceil +2,
\end{align*}
which implies
\[
\frac{d\ell - \delta q}{T} \le \frac{(d+\tau)(\ell - 2\alpha-2\beta)}{2d/\tau + 2} = \tau(\ell/2-\alpha-\beta).
\]
This gives us a neural network $g_\ell \in \NN(W(\ell),L(\ell))$ with width
\begin{align*}
W(\ell) &\le C \max \left\{ \frac{b^{\delta q/2}}{S \sqrt{T}}, b^{(d\ell -\delta q)/T}, b^{\ell/L_0 }, W_0 2^{W_0} \right\} \\
&\le C \max \left\{ b^{d\alpha+\tau(\alpha+\beta-\ell/2)}, b^{\tau(\ell/2-\alpha-\beta)}, b^{\alpha/2}, \alpha b^{\alpha/2} \right\} \\
&\le C \left( b^{d\alpha+\tau(\alpha+\beta-\ell/2)} + b^{\tau(\ell/2-\alpha-\beta)} + \alpha b^{\alpha/2} \right),
\end{align*}
and depth
\[
L(\ell) \le C(ST+L_0) \le C \left(b^{d\beta} + \beta \right) \le Cb^{d\beta}.
\]
Moreover, by Proposition \ref{app of poly} and inequality (\ref{norm bound for coe}), the approximation error satisfies
\begin{equation}\label{error 2}
\begin{aligned}
\|f_\ell-g_\ell\|_{L^p(\Omega_{\ell,\epsilon})} &\le C \left(b^{\delta q/p -d\ell/p-\delta} + 4^{-W_0L_0}\right) b^{d\ell/q-s\ell} \\
&\le C b^{-2s(\alpha+\beta)} \left( b^{\eta (\ell -2\alpha-2\beta)} +b^{-s\ell} \right),
\end{aligned}
\end{equation}
where $\eta:= d/q-d/p-s+\tau(1/q-1/p)<0$ by condition (\ref{tau}) and we use (\ref{small error}) in the last inequality. By Proposition \ref{basic constr}, we can construct the sum $G_3 = \sum_{\ell\in\cL_3} g_\ell$ in a parallel way, so that $G_3\in \NN(W_3,L_3)$ with
\begin{align*}
W_3 &= \sum_{\ell \in \cL_3} W(\ell) \le C \sum_{\ell=2\alpha+2\beta+1}^{2\alpha+2\beta+\lfloor 2d\alpha/\tau \rfloor} \left( b^{d\alpha+\tau(\alpha+\beta-\ell/2)} + b^{\tau(\ell/2-\alpha-\beta)} + \alpha b^{\alpha/2} \right) \le Cb^{d\alpha}, \\
L_3 &= \max_{\ell \in \cL_3} L(\ell) \le  C b^{d\beta}.
\end{align*}

\textbf{Case 4:} $\ell \in \cL_4=\{\ell \in \cL: 2\alpha+2\beta + 2d\alpha/\tau < \ell \le \ell^* \}$. By condition (\ref{tau}),
\[
\ell^* - 2\alpha -2\beta \le 2(\kappa-1)(\alpha+\beta)<2d(\alpha+\beta)/\tau.
\]
Hence, we can choose $\delta(\ell)$ as in Case 3 so that $\delta>0$ and $\delta q < d\ell$. We apply Proposition \ref{app of poly} Part (1) with parameters
\begin{align*}
S(\ell) &= \left\lceil b^{\delta q/2} \right\rceil, \\
T(\ell) &= \left\lceil 2d/\tau + 2 \right\rceil \left\lceil \tau(\ell/2-\alpha-\beta)-d\alpha \right\rceil.
\end{align*}
This gives us a neural network $g_\ell \in \NN(W(\ell),L(\ell))$. Since
\begin{align*}
\frac{d\ell - \delta q}{T} &\le \frac{(d+\tau)(\ell - 2\alpha-2\beta)}{(2d/\tau + 2) \lceil \tau(\ell/2-\alpha-\beta)-d\alpha \rceil} = \frac{\tau(\ell/2-\alpha-\beta)}{\lceil \tau(\ell/2-\alpha-\beta)-d\alpha \rceil} \\
&\le 1+ \frac{d\alpha}{\lceil \tau(\ell/2-\alpha-\beta)-d\alpha \rceil} \le 1+d\alpha,
\end{align*}
the width $W(\ell)$ satisfies
\begin{align*}
W(\ell) &\le C \max \left\{ \frac{b^{\delta q/2}}{S \sqrt{T}}, b^{(d\ell -\delta q)/T}, b^{\ell/L_0 }, W_0 2^{W_0} \right\} \\
&\le C \max \left\{ 1, b^{d\alpha+1}, b^{\alpha/2}, \alpha b^{\alpha/2} \right\} \\
&\le Cb^{d\alpha}.
\end{align*}
The depth $L(\ell)$ can be bounded as
\[
L(\ell) \le C(ST+L_0) \le C \left(b^{d(\alpha+\beta)-\tau(\ell/2-\alpha-\beta)}(\tau(\ell/2-\alpha-\beta)-d\alpha) + \beta \right).
\]
By Proposition \ref{app of poly} and inequality (\ref{norm bound for coe}), we also get the approximation error (\ref{error 2}) in this case. Using Proposition \ref{basic constr}, we can construct the sum $G_4 = \sum_{\ell\in\cL_4} g_\ell$ in a sequential way, so that $G_4\in \NN(W_4,L_4)$ with
\begin{align*}
W_4 &= \max_{\ell \in \cL_4} W(\ell) +2d+2 \le Cb^{d\alpha}, \\
L_4 &= \sum_{\ell \in \cL_4} L(\ell) \le C \sum_{\ell = 2\alpha +2\beta +\lceil 2d\alpha/\tau \rceil}^{\ell^*} \left(b^{d(\alpha+\beta)-\tau(\ell/2-\alpha-\beta)}(\tau(\ell/2-\alpha-\beta)-d\alpha) + \beta \right) \\
&\le C \sum_{j = \lceil 2d\alpha/\tau \rceil}^{\lfloor 2d(\alpha+\beta)/\tau \rfloor} \left(b^{d(\alpha+\beta)-\tau j/2}(\tau j/2-d\alpha) + \beta \right) \le Cb^{d\beta},
\end{align*}
where we use $\ell^* - 2\alpha -2\beta <2d(\alpha+\beta)/\tau$ and the last series is dominated by the term corresponding to $j=\lceil 2d\alpha/\tau \rceil$.

Finally, we construct the desired network as 
\[
g_{\alpha,\beta} = \sum_{i=1}^4 G_i = \sum_{\ell=0}^{\ell^*} g_\ell \in \NN(W,L),
\]
whose width $W\le Cb^{d\alpha}$ and depth $L\le Cb^{d\beta}$ by the above analysis. Since $\Omega_{\ell^*,\epsilon} \subseteq \Omega_{\ell,\epsilon}$ for $\ell \le \ell^*$, the fact that inequality (\ref{error 1}) holds for $\ell \in \cL_1\cup \cL_2$ implies
\[
\sum_{\ell=0}^{2\alpha+2\beta} \| f_\ell - g_\ell\|_{L^p(\Omega_{\ell^*,\epsilon})} \le C b^{-2s(\alpha+\beta)} \sum_{\ell=0}^{2\alpha+2\beta} \left( b^{-2\alpha-2\beta+\ell} +b^{-s\ell} \right) \le C b^{-2s(\alpha+\beta)}.
\]
Similarly, using inequality (\ref{error 2}), we get
\[
\sum_{\ell=2\alpha+2\beta+1}^{\ell^*} \| f_\ell - g_\ell\|_{L^p(\Omega_{\ell^*,\epsilon})} \le C b^{-2s(\alpha+\beta)} \sum_{\ell=2\alpha+2\beta+1}^{\ell^*} \left( b^{\eta (\ell -2\alpha-2\beta)} +b^{-s\ell} \right) \le C b^{-2s(\alpha+\beta)},
\]
since $\eta<0$. As a consequence, the approximation error can be bounded as
\begin{align*}
\|f-g_{\alpha,\beta} \|_{L^p(\Omega_{\ell^*,\epsilon})} &\le \sum_{\ell=0}^{\ell^*} \| f_\ell - g_\ell\|_{L^p(\Omega_{\ell^*,\epsilon})} + \sum_{\ell=\ell^*+1}^\infty \|f_\ell\|_{L^p(\Omega_{\ell^*,\epsilon})} \\
&\le C b^{-2s(\alpha+\beta)} + C \sum_{\ell=\ell^*+1}^\infty  b^{(d/q-d/p-s)\ell} \\
&\le C \left( b^{-2s(\alpha+\beta)} + b^{2(d/q-d/p-s)\kappa(\alpha+\beta)} \right) \\
&\le C b^{-2s(\alpha+\beta)},
\end{align*}
where we use (\ref{norm bound for fl}) in the second inequality and $(d/q-d/p-s)\kappa=-s$ in the last inequality.
\end{proof}

In order to prove Theorem \ref{app of sobolev}, we need to remove the trifling region from Proposition \ref{app of sobolev part} by using ideas from \citet{shen2022optimal,lu2021deep,siegel2023optimal}. Specifically, we will follow the construction in \citet{siegel2023optimal}, which uses different bases $b_i$ to create multiple approximators (by using Proposition \ref{app of sobolev part}). The bases can be chosen in such a way that they create minimally overlapping trifling regions and the median of the approximators has the desired accuracy on the whole domain $\Omega$.

\begin{proof}[Proof of Theorem \ref{app of sobolev}]
Without loss of generality, we assume that $f \in \cW^{s,q}([0,1]^d)$ has been normalized so that $\|f\|_{\cW^{s,q}([0,1]^d)} \le 1$. The case $p<q$ can be reduced to the case $p=q$ by using the inequality $\|f-g\|_{L^p([0,1]^d)} \le \|f-g\|_{L^q([0,1]^d)}$ for $p<q$. So, we can also assume that $1\le q\le p\le \infty$. 

In order to remove the trifling region in Proposition \ref{app of sobolev part}, we make use of different bases $b$. Let $k$ be the smallest integer such that $2^k \ge 2d+2$ and let $b_i$ be the $i$-th prime number for $i=1,\dots,2^k$. Thus, $k$ and $b_i$ only depend on the dimension $d$. To complete the proof, it is sufficient to show that, for any integers $m,n \ge b_{2^k}$, there exists $g\in \NN(W,L)$ with $W\le Cm^d$ and $L\le Cn^d$ such that
\[
\|f-g\|_{L^p(\Omega)} \le C (mn)^{-2s},
\]
where $\Omega = [0,1)^d$ as before.

For $i=1,\dots, 2^k$, we denote $\alpha_i = \lfloor \log_{b_i} m \rfloor$ and $\beta_i = \lfloor \log_{b_i} n \rfloor$. Thus, we have the simple inequalities $b_i^{-1} m \le b_i^{\alpha_i} \le m$ and $b_i^{-1} n \le b_i^{\beta_i} \le n$. In order to apply Proposition \ref{app of sobolev part}, we let $\ell_i^* = \lfloor 2\kappa (\alpha_i+\beta_i) \rfloor$ where $\kappa$ is defined as in Proposition \ref{app of sobolev part}. Notice that the following numbers are all distinct
\[
\cA:= \bigcup_{i=1}^{2^k} \left\{ \frac{1}{b_i^{\ell_i^*}},\dots, \frac{b_i^{\ell_i^*}-1}{b_i^{\ell_i^*}} \right\},
\]
since $b_i$ are all pairwise relatively prime. We choose $\epsilon>0$ small enough so that 
\[
\epsilon < \min_{u\neq v\in \cA} |u-v|.
\]
This choice of $\epsilon$ has the property that any element of $[0,1)$ is contained in at most one of the sets
\begin{equation}\label{bad sets}
[j b_i^{-\ell_i^*}-\epsilon, j b_i^{-\ell_i^*}), \quad j=1,\dots, b_i^{\ell_i^*} -1 \mbox{ and } i=1,\dots,2^k.
\end{equation}
Thus, if we let $\Omega_{\ell_i^*,\epsilon}$ denote the good region (\ref{good region}) at level $\ell_i^*$ with base $b_i$, then, for any $x\in \Omega$, we have $x\notin \Omega_{\ell_i^*,\epsilon}$ for at most $d$ different values $i$, because each coordinate of $x$ can be contained in at most one set from (\ref{bad sets}). In other words, the following set 
\[
\cI(x) := \{i: x\in \Omega_{\ell_i^*,\epsilon}\}
\]
has at least $2^k-d \ge 2^{k-1}+1$ elements, since $2^k \ge 2d+2$.

For $i=1,\dots, 2^k$, by applying Proposition \ref{app of sobolev part} with parameters $\alpha_i, \beta_i$, the base $b_i$ and the above $\epsilon$, we get $g_i \in \NN(W_i,L_i)$ with $W_i \le Cb_i^{d\alpha_i}\le Cm^d$ and $L_i\le Cb_i^{d\beta_i} \le Cn^d$ such that
\[
\| f- g_i \|_{L^p(\Omega_{\ell_i^*,\epsilon})} \le C b_i^{-2s(\alpha_i + \beta_i)} \le C(mn)^{-2s}.
\]
Let $\psi_{2^{k-1}} \in \NN(2^{k+2},k(k+1)/2)$ be the network in Proposition \ref{select median} that selects the $2^{k-1}$-largest value from $2^k$ values. We construct the desired network $g$ as
\[
x \to 
\begin{pmatrix}
g_1(x) \\
\vdots \\
g_{2^k}(x)
\end{pmatrix}
\to \psi_{2^{k-1}}(g_1(x),\dots, g_{2^k}(x)).
\]
Then, by Proposition \ref{basic constr}, we have $g\in \NN(W,L)$ with 
\begin{align*}
W &\le \max \left\{ \sum_{i=1}^{2^k} W_i, 2^{k+2} \right\} \le C m^d, \\
L &\le \max_{1\le i\le 2^k} L_i + k(k+1)/2 \le C n^d.
\end{align*}
It remains to estimate the approximation error. For each $x\in\Omega$, since $|\cI(x)| \ge 2^{k-1}+1$, the $2^{k-1}$-largest element of $\{g_1(x),\dots,g_{2^k}(x)\}$ must be both larger and smaller than some element of $\{g_i(x):i\in \cI(x)\}$. In other words, 
\[
\min_{i\in \cI(x)} g_i(x) \le g(x) \le \max_{i\in \cI(x)} g_i(x),
\]
which implies 
\[
|f(x) -g(x)| \le \max_{i\in \cI(x)} |f(x) - g_i(x)|. 
\]
When $p=\infty$, we finish the proof by noticing that the right hand side is bounded by $C(mn)^{-2s}$ by the definition of $\cI(x)$. When $p<\infty$, we have
\begin{align*}
\int_{\Omega} |f(x) -g(x)|^p dx &\le \int_{\Omega} \sum_{i\in \cI(x)} |f(x) - g_i(x)|^p dx \\
&\le \sum_{i=1}^{2^k} \| f- g_i \|_{L^p(\Omega_{\ell_i^*,\epsilon})}^p \\
&\le C (mn)^{-2sp},
\end{align*}
which completes the proof by taking $p$-th roots.
\end{proof}

\begin{remark}\label{besov detail}
Theorem \ref{app of besov} can be proven in the same way as above with minor modifications. The main differences are that we choose $k= \lfloor s \rfloor +1$ and the Bramble-Hilbert lemma needs to be replaced by the analogous inequality $\|\Pi_k^0(f) - f\|_{L^q(\Omega)} \le C\|f\|_{\cB^s_{q,r}(\Omega)}$, which follows from well-known bound for piecewise polynomial approximation of Besov functions \citep[Section 3]{devore1988interpolation}. Additionally, the sub-additivity (\ref{sub-add}) should be replaced by the corresponding inequality for Besov spaces
\[
\sum_{\Bi \in I_\ell} |f|_{\cB^s_{q,r}(\Omega_\Bi^\ell)}^q \le C |f|_{\cB^s_{q,r}(\Omega)}^q.
\]
We refer the reader to \citet{devore1993besov} for reference.
\end{remark}

\section*{Acknowledgments}

The work described in this paper was partially supported by National Natural Science Foundation of China under Grant 12371103. We thank the referees for their helpful comments and suggestions on the paper.

\appendix

\section{Proofs of technical results}

\subsection{Proof of Lemma \ref{PWL rep}}

For the first part, we let $t_1<t_2<\dots<t_m$ with $1\le m\le n$ be the breakpoints of $g$ (if $g$ is affine, let $m=1$ and $t_1=0$). Observe that $g$ can be written as 
\[
g(t) = c- a_0\sigma(-t+t_1) + \sum_{i=1}^m a_i\sigma(t-t_i),
\]
where $a_0$ and $a_1$ are the left and right derivatives at $t_1$, the remained $a_i$ give the jump in derivative at other breakpoints and $c$ is set to match the value at $0$. Hence, we have $g\in \NN(m+1,1)$. If $a_0=0$, then $g\in \NN(m,1)$.

The second part follows from \citet[Theorem 3.1]{daubechies2021nonlinear}, which showed that $g=f$ on $[\alpha,\beta]$ for some $f\in \NN(6W+2,2 \lceil \frac{n}{6W^2} \rceil) \subseteq \NN(6W+2,2L)$ when $[\alpha,\beta]=[0,1]$. We can extend their result on $[0,1]$ to any bounded interval $[\alpha,\beta]$ by applying an affine map on the input. Note that one can also extend the bounded interval to $(-\infty,\infty)$ by using a slightly larger network width.

\subsection{Proof of Lemma \ref{bit extraction}}

We modify the construction in \citet[Lemma 13]{bartlett2019nearly}. We partition $[0,1)$ into $2^m$ intervals $[j2^{-m},(j+1)2^{-m})$, $j=0,1,\dots,2^m-1$. The indicator function of $[j2^{-m},(j+1)2^{-m})$ can be approximated by the following piecewise linear function
\[
g_j(t) = 
\begin{cases}
0, & t\le j2^{-m} -\epsilon \mbox{ or } t\ge (j+1)2^{-m}, \\
1+\epsilon^{-1}(t-j2^{-m}), & j2^{-m} -\epsilon < t< j2^{-m},\\
1, & j2^{-m} \le  t\le  (j+1)2^{-m} -\epsilon, \\
-\epsilon^{-1}(t-(j+1)2^{-m}), & (j+1)2^{-m} -\epsilon < t < (j+1)2^{-m},
\end{cases}
\]
where we choose $\epsilon<2^{-n}$. By Lemma \ref{PWL rep}, we have $g_j\in \NN(4,1)$.

Observe that $x_i$ can be computed by adding the corresponding indicator function. For instance, $x_1 = \sum_{j=2^{m-1}}^{2^m-1} g_j(x)$. Furthermore, $\Bin 0.x_{m+1}\cdots x_n=2^m x - \sum_{i=1}^m 2^{m-i}x_i$. We can construct the desired network $f_{m,n}$ as follows
\[
x \to 
\begin{pmatrix}
g_0(x) \\
\vdots \\
g_{2^m-1}(x) \\
x
\end{pmatrix}
\to
\begin{pmatrix}
x_1 \\
\vdots \\
x_m \\
x
\end{pmatrix}
\to 
\begin{pmatrix}
x_1 \\
\vdots \\
x_m \\
\Bin 0.x_{m+1}\cdots x_n
\end{pmatrix}.
\]
Note that the last two maps are affine. Hence, $f_{n,m}\in \NN(2^{m+2}+1,1)$. 

To construct the network $f_{n,m,L}$ with $L\le m$, we can apply $f_{n,\lceil m/L \rceil}$ to the last component $L-1$ times and then apply $f_{n,m - (L-1)\lceil m/L \rceil}$ to extract $m$ bits:
\begin{align*}
x &\to
\begin{pmatrix}
\Bin x_1\cdots x_{\lceil m/L \rceil}.0 \\
\Bin 0.x_{\lceil m/L \rceil+1}\cdots x_n
\end{pmatrix}
\to 
\begin{pmatrix}
\Bin x_1\cdots x_{2\lceil m/L \rceil}.0 \\
\Bin 0.x_{2\lceil m/L \rceil+1}\cdots x_n
\end{pmatrix} \\
&\to \cdots \to 
\begin{pmatrix}
\Bin x_1\cdots x_{(L-1)\lceil m/L \rceil}.0 \\
\Bin 0.x_{(L-1)\lceil m/L \rceil+1}\cdots x_n
\end{pmatrix} \\
&\to 
\begin{pmatrix}
\Bin x_1\cdots x_m.0 \\
\Bin 0.x_m\cdots x_n
\end{pmatrix} \in \NN(2^{\lceil m/L \rceil+2}+2,L).
\end{align*}
Note that, for $j\le i$, $\Bin x_1\cdots x_i.0$ is a linear combination of $x_j,\dots,x_i$ and $\Bin x_1\cdots x_j.0$. Finally, for $L>m$, we let $f_{n,m,L} = f_{n,m,m} \in \NN(10,m) \subseteq \NN(10,L)$.

\subsection{Proof of Lemma \ref{app partition}}

The result is trivial for $\ell =0$. So we let $\ell \ge 1$ in the following. Let us first consider the one-dimensional case $d=1$. For each $m\in [1,\ell]$, we define the piecewise linear function
\[
g_m(t):= 
\begin{cases}
0, &t\le b^{-m}-\epsilon,\\
j+\epsilon^{-1}(t-jb^{-m}), &jb^{-m}-\epsilon<t\le jb^{-m}, \mbox{ for }j=1,\dots,b^m-1,\\
j, &jb^{-m}<t\le (j+1)b^{-m}-\epsilon, \mbox{ for }j=1,\dots,b^m-1,\\
b^m-1, &t>1-\epsilon.
\end{cases}
\]
We show the graph of $g_m$ in Figure \ref{graph of gm}. It is easy to see that $g_m$ has $2b^m-1$ pieces and $g_m\in \NN(2b^m-2,1)$ by Lemma \ref{PWL rep}. 
\begin{figure}[htbp]
\centering
\includegraphics[width=0.95\linewidth]{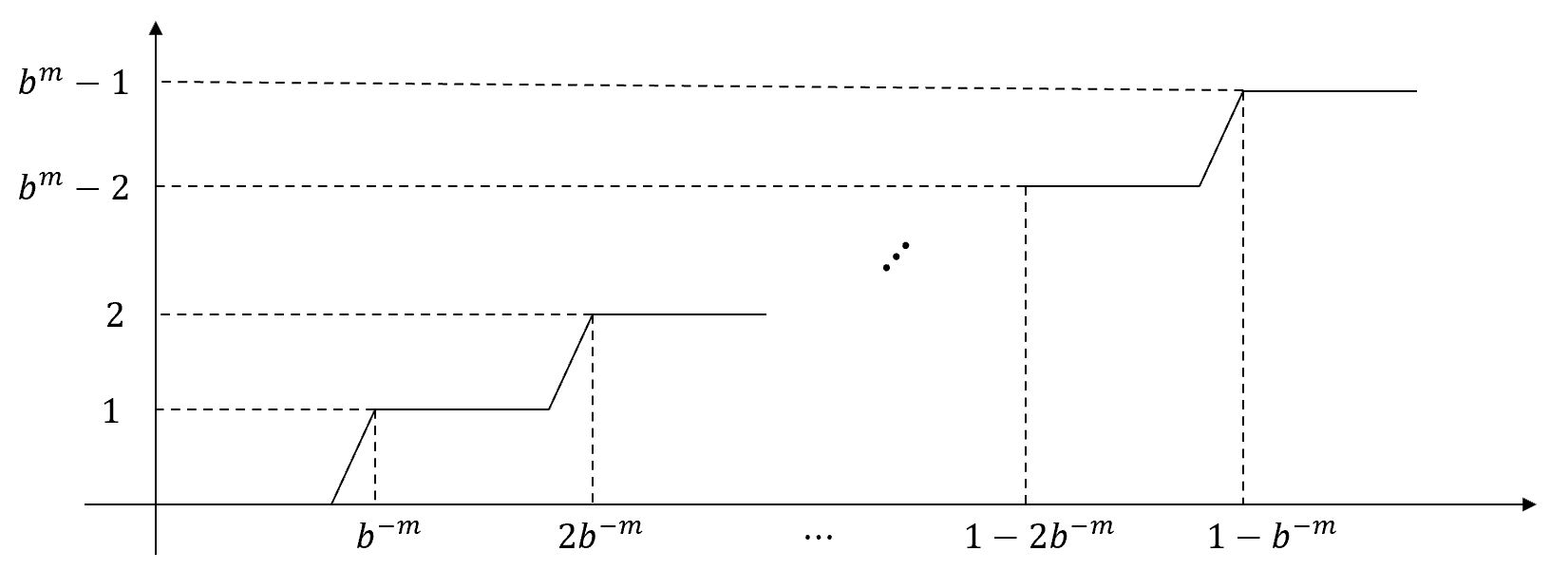}
\caption{The graph of function $g_m(t)$ in the proof of Lemma \ref{app partition}.}
\label{graph of gm}
\end{figure}

For any $t\in [ib^{-\ell},(i+1)b^{-\ell}-\epsilon)$, we have the $b$-adic representation
\[
t = \sum_{k=1}^\ell a_k b^{-k} + c,\quad i= \sum_{k=1}^\ell a_k b^{\ell -k},
\]
where $a_k\in \{0,\dots,b-1\}$ and $c\in [0,b^{-\ell} -\epsilon)$. Let $r:=\lfloor \ell/m \rfloor \ge 1$ and $s:= \ell - rm \in [0,m-1]$. We consider the following iteration with initialization $p_0=0$ and $t_0=t$:
\[
p_{n+1} = b^m p_n + g_m(t_n), \quad t_{n+1} = b^m t_n - g_m(t_n),\quad \mbox{for }n=0,\dots,r-1.
\]
Since $\sum_{k=m+1}^\ell a_kb^{-k}+c<b^{-m}-\epsilon$, we have $p_1 = g_m(t)= \sum_{k=1}^m a_k b^{m-k}$ and $t_1 = b^mt - g_m(t) = \sum_{k=m+1}^\ell a_kb^{m-k}+b^m c$. Inductively, one can check that
\[
p_n = \sum_{k=1}^{nm} a_k b^{nm-k},\quad t_n = \sum_{k=nm+1}^\ell a_kb^{nm-k}+b^{nm} c,\quad \mbox{for }n=1,\dots,r.
\]
If $s=0$, we already obtain the index $i = p_r$. If $s\neq 0$, we need one more iteration step
\[
p_{r+1} = b^sp_r + g_s(t_r).
\]
Since $b^{rm}c<b^{-s} -b^{rm}\epsilon<b^{-s} -\epsilon$, we have $g_s(t_r) = \sum_{k=rm+1}^\ell a_kb^{\ell-k}$ and $p_{r+1} = \sum_{k=1}^\ell a_k b^{\ell -k}=i$. Thus, we have given an iteration method to compute the index $i$ for $t\in [ib^{-\ell},(i+1)b^{-\ell}-\epsilon)$. In addition, for $t\in [1-\epsilon,1]$, the above iteration gives $p_n=b^{nm}-1$ for $n=1,\dots,r$. If $s\neq 0$, we get $p_{r+1} = b^\ell -1$. Hence, we also compute the index correctly for $t\in [1-\epsilon,1]$.

Next, we construct a neural network to implement the above iteration. We begin with the affine map $t \to (0,t)^\top $. Then, one iteration step can be implemented by composing with the following map
\[
\begin{pmatrix}
p\\
t
\end{pmatrix}
\to
\begin{pmatrix}
\sigma(p) \\
\sigma(t) \\
g_m(t)
\end{pmatrix}
\to
\begin{pmatrix}
b^m \sigma(p)+ g_m(t) \\
b^m \sigma(t) - g_m(t)
\end{pmatrix}
\in \NN(2b^m,1).
\]
If $s\neq 0$, we simply replace $g_m(t)$ by $g_s(t)$ in the $(r+1)$-th iteration step. After the last step of the iteration, we compose with the affine map which selects the first coordinate to get the desired network $q_1 \in \NN_{1,1}(2b^m,\lceil \ell/m \rceil)$. For any $L\in \bN$, if we choose $m= \lceil \ell/L \rceil$, then $q_1 \in \NN_{1,1}(2b^{\lceil \ell/L \rceil},L)$.

Finally, for higher dimensional case $d\ge 2$, we notice that 
\[
q_d(x) = \sum_{j=1}^d b^{\ell(j-1)} q_1(x_j),\quad \forall x\in \Omega_{\Bi,\epsilon}^\ell.
\]
By Proposition \ref{basic constr}, we have $q_d \in \NN_{d,1}(2db^{\lceil \ell/L \rceil},L)$.

\subsection{Proof of Theorem \ref{rep lower bound}}\label{proof of rep lower bound}

We begin with some notations for the neural network class $\NN_{1,1}(W,L)$. For integer $\ell \in [1,L+1]$, we use $P_\ell$ to denote the number of parameters (weights and biases) \emph{up to layer} $\ell$. Thus, $P:=P_{L+1} = (L-1)W^2 + (L+2)W+1$ is the number of parameters in the network. It is easy to see that $W\ell \le P_\ell \le 2W^2 \ell$. We use $f_a\in \NN(W,L)$ to denote the neural network parameterized by $a\in \bR^P$. 

Observe that, for any subset $S$ of $\{1,\dots,\min\{N,M\}\}$, there exists $x=(x_1,\dots,x_N)^\top \in \cS_{N,M}$ such that $x_i>0$ if and only if $i\in S$. By assumption, the function class $\NN(W,L)$ must shatter the set $\{1,\dots,\min\{N,M\}\}$ and hence
\begin{equation}\label{temp pdim}
\min\{N,M\} \le \Pdim(\NN(W,L)) \le C W^2L^2 \log(WL),
\end{equation}
where the second inequality is from (\ref{Pdim bound}). Thus, when $2N \ge M \ge N/2$, i.e. $p=1$, we get the desired bound. For $p>1$, we cannot use the pseudo-dimension directly, but we can apply an argument similar to the pseudo-dimension bound in \citet{bartlett2019nearly}. The main technical tool is the following form of Warren's Theorem from \citet[Lemma 17]{bartlett2019nearly} and \citet[Theorem 8.3]{anthony2009neural}.

\begin{lemma}\label{Warren}
Suppose $P\le N$ and let $f_1,\dots,f_N$ be polynomials of degree at most $D$ in $P$ variables. Then, the number of possible sign vectors attained by the polynomials can be bounded as
\[
|\{(\sgn(f_1(a)),\dots, \sgn(f_N(a))):a\in \bR^P \}|\le 2(2eND/P)^P,
\]
where $\sgn(t) =1$ if $t>0$ and $\sgn(t)=0$ otherwise.
\end{lemma}

We first consider the case that $N/2 > M \ge C_0 N^{1/p}$. We are going to estimate the number of sign patterns that the neural network can output on the input set $\{1,\dots,N\}$. Specially, we define 
\begin{align*}
s(a) &= (\sgn(f_a(1)),\dots, \sgn(f_a(N))) \in \{0,1\}^N,\\
K &= |\{s(a):a\in \bR^P\}|.
\end{align*}
Note that we can assume that $P\le N$, because otherwise $4W^2L \ge P >N > 2M$ already implies the desired result. So, one can apply Lemma \ref{Warren} in the following analysis. To upper bound $K$, we partition $\bR^P$ into regions where $f_a(i)$, $i=1,\dots,N$, are polynomials of $a$. This can be done by using the method presented in the proof of \citet[Theorem 7]{bartlett2019nearly}. By using Lemma \ref{Warren}, they constructed iteratively a sequence of refined partition $\cA_0,\dots, \cA_L$ with the following two properties:
\begin{enumerate}[label=\textnormal{\arabic*.},parsep=0pt]
\item $\cA_0 = \bR^P$ and for $\ell =1,\dots, L$,
\begin{equation}\label{bound for A}
\frac{|\cA_\ell|}{|\cA_{\ell-1}|} \le 2 \left( \frac{2eNW\ell}{P_\ell} \right)^{P_\ell}.
\end{equation}

\item For each $\ell =1,\dots, L+1$, each element $A$ of $\cA_{\ell-1}$, each input $i=1,\dots,N$, and each neuron $u$ in the $\ell$-th layer, when $a$ varies in $A$, the net input to $u$ is a fixed polynomial function in $P_\ell$ variables of $a$, with total degree at most $\ell$. 
\end{enumerate}
In particular, for each $A\in \cA_L$, $f_a(i)$ is a polynomial of $a\in A$ with degree at most $L+1$, since we do not have activation in the last layer. By Lemma \ref{Warren}, we get 
\[
|\{s(a):a\in A\}| \le 2 \left( \frac{2eN(L+1)}{P_{L+1}}\right)^{P_{L+1}}.
\]
Applying the bound (\ref{bound for A}) iteratively gives
\[
|\cA_L| \le \prod_{\ell=1}^L 2 \left( \frac{2eNW\ell}{P_\ell} \right)^{P_\ell}.
\]
As a consequence,
\begin{align*}
K &\le \sum_{A\in \cA_L} |\{s(a):a\in A\}| \le \prod_{\ell=1}^{L+1} 2 \left( \frac{2eNW\ell}{P_\ell} \right)^{P_\ell} \\
&\le 2^{L+1} (2eN)^{\sum_{\ell=1}^{L+1} P_\ell} \le (4eN)^{W^2(L+1)(L+2)},
\end{align*}
where we use $W\ell \le P_\ell \le 2W^2\ell$ in the last two inequalities. On the other hand, $\NN(W,L)$ can match the values of any element in $\cS_{N,M}$ by assumption. Since $\cS_{N,M}$ contains every indicator function of every subset of $\{1,\dots,N\}$ of size $M$, we have
\[
\binom{N}{M} \le K \le (4eN)^{W^2(L+1)(L+2)}.
\]
Taking logarithms shows that
\[
M\log(N/M) \le 6W^2L^2 \log(4eN).
\]
Since $N/2 > M \ge C_0 N^{1/p}$, we have
\begin{align*}
M(1+\log(N/M)) &\le eM\log(N/M) \le 6eW^2L^2 \log(4eN) \\
&\le 6eW^2L^2 \log(4eC_0^{-p}M^p) \\
&\le C_p W^2L^2 \log(WL),
\end{align*}
where we use (\ref{temp pdim}) in the last inequality.

For the case $2N< M\le C_0 N^p$, we can assume that $P\le M$, because otherwise $4W^2L \ge P >M > 2N$ already implies the desired result. We consider a slightly different vector of sign pattern:
\[
s(a) = (\sgn(f_a(i)-j))_{i\in\{1,\dots,N\}, j\in\{0,\dots,M-1\}} \in \{0,1\}^{NM}.
\]
The only difference is that we now consider $NM$ piecewise polynomials $a \mapsto f_a(i)-j$ indexed by $(i,j)$, rather than $N$ piecewise polynomials $a\mapsto f_a(i)$ indexed by $i$. We can upper bound $K$, defined through the new sign pattern $s(a)$, in a similar manner as before and obtain
\[
K \le \prod_{\ell=1}^{L+1} 2 \left( \frac{2eNMW\ell}{P_\ell} \right)^{P_\ell} \le (4eNM)^{W^2(L+1)(L+2)}.
\]
Notice that the set $\cS_{N,M}$ contains all vectors whose first $N-1$  coordinates are arbitrary integers in $\{0,1,\dots,\lfloor M/N \rfloor\}$ and whose last coordinate is chosen to make the $\ell^1$ norm equal to $M$. By assumption, $\cS_{N,M}$ can be represented by $\NN(W,L)$, which implies
\[
(\lfloor M/N \rfloor+1)^{N-1} \le K \le (4eNM)^{W^2(L+1)(L+2)}.
\]
Taking logarithms and calculating as before, we get the desired bound.

\subsection{Proof of Lemma \ref{app of product}}

Following the construction in \citet[Lemma 5.1]{lu2021deep}, we define a set of sawtooth functions $T_i:\bR \to [0,1]$ by iteration $T_i = T_{i-1} \circ T_1$ for $i=2,3,\dots$, and
\[
T_1(x) := 
\begin{cases}
2x, & x\in [0,1/2],\\
2(1-x), & x\in (1/2,1],\\
0, & x\notin [0,1].
\end{cases}
\]
It is easy to see that $T_i$ has $2^{i-1}$ sawteeth and $T_i\in \NN(2^i,1)$ by Lemma \ref{PWL rep}. Consequently, the symmetric function $x \mapsto T_i(|x|)$ is in $\NN(2^{i+1},1)$.

Let us first consider the approximation of the square function on $[-1,1]$. For any $s\in \bN$, let $h_s:[-1,1] \to [0,1]$ be the continuous piecewise linear function with breakpoints $h_s(j/2^s)=(j/2^{s})^2$ for $j\in \bZ \cap [-2^s,2^s]$. Note that $h_s$ is a symmetric convex function. One can check that (see \citet[Lemma 5.1]{lu2021deep}), for any $x\in [-1,1]$, $0\le h_s(x) - x^2 \le 2^{-2s-2}$ and 
\[
h_s(x) = |x| - \sum_{i=1}^s 4^{-i} T_i(|x|).
\]
For any integers $k\ge 4$ and $L\ge 1$, we define the network $g_{k,L}$ by
\begin{align*}
x &\to 
\begin{pmatrix}
T_1(|x|) \\
\vdots \\
T_k(|x|) \\
|x|
\end{pmatrix}
\to 
\begin{pmatrix}
T_{k+1}(|x|) \\
\vdots \\
T_{2k}(|x|) \\
|x| - \sum_{i=1}^k 4^{-i} T_i(|x|)
\end{pmatrix} 
\to \cdots \to 
\begin{pmatrix}
T_{(L-1)k+1}(|x|) \\
\vdots \\
T_{Lk}(|x|) \\
|x| - \sum_{i=1}^{(L-1)k} 4^{-i} T_i(|x|)
\end{pmatrix} \\
&\to |x| - \sum_{i=1}^{Lk} 4^{-i} T_i(|x|),
\end{align*}
which satisfies $g_{k,L}(x) = h_{kL}(x)$ for $x\in [-1,1]$. In this construction, the width of the first layer is $2+\sum_{i=1}^k 2^{i+1} \le 2^{k+2}$ and the width of remained layers is $k2^k+1$. Hence, we have $g_{k,L} \in \NN(k2^k+1,L)$ because $k\ge 4$.

Observing that $xy=2(\frac{x+y}{2})^2 - \frac{x^2+y^2}{2}$, we define the desired network $f_{k,L}\in \NN(3k2^k+3,L)$ by 
\[
f_{k,L}(x,y) := 2g_{k,L}\left(\frac{x+y}{2}\right) - \frac{g_{k,L}(x)+g_{k,L}(y)}{2}.
\]
For $x,y\in[-1,1]$, we have $f_{k,L}(x,y) \ge - \frac{g_{k,L}(x)+g_{k,L}(y)}{2} \ge -1$ and $f_{k,L}(x,y) \le g_{k,L}\left(\frac{x+y}{2}\right) \le 1$, since $g_{k,L} = h_{kL}:[-1,1]\to [0,1]$ is convex. Furthermore,
\begin{align*}
f_{k,L}(x,y) - xy &= 2 \left( h_{kL}\left(\frac{x+y}{2}\right) - \left(\frac{x+y}{2}\right)^2 \right) - \frac{h_{kL}(x) - x^2}{2} - \frac{h_{kL}(y) - y^2}{2} \\
&\in [-2^{-2kL-2}, 2^{-2kL-1}],
\end{align*}
where we use $0\le h_{kL}(x) - x^2 \le 2^{-2kL-2}$.

\bibliographystyle{myplainnat}
\bibliography{references}
\end{document}